\documentclass{article}
\usepackage{microtype}
\usepackage{graphicx}
\usepackage{booktabs} 

\usepackage{hyperref}

\usepackage[accepted]{icml2018}

\usepackage{amsmath}
\usepackage{amsthm}
\usepackage{amssymb}
\usepackage{bbm}
\usepackage{graphicx}
\usepackage{caption}
\usepackage{subcaption}
\usepackage{afterpage}
\usepackage{mathtools}
\usepackage{xspace}
\usepackage{dsfont}

\def\One{\mathds{1}}
\def\E{\mathbb{E}}
\def\R{\mathbb{R}}
\def\passive{Passive}
\def\dbalw{DBALw\xspace}
\def\dbalwm{DBALwm\xspace}
\def\idbal{IDBAL\xspace}

\theoremstyle{plain}
\newtheorem{thm}{\protect\theoremname}
  \theoremstyle{definition}
  \newtheorem{defn}[thm]{\protect\definitionname}
  \theoremstyle{remark}
  
  \theoremstyle{plain}
  \newtheorem{prop}[thm]{\protect\propositionname}
  \theoremstyle{plain}
  \newtheorem{lem}[thm]{\protect\lemmaname}
  \theoremstyle{plain}
  \newtheorem{cor}[thm]{\protect\corollaryname}
  \theoremstyle{plain}
  \newtheorem{fact}[thm]{\protect\factname}
  \theoremstyle{plain}
\providecommand{\corollaryname}{Corollary}
  \providecommand{\definitionname}{Definition}
  \providecommand{\lemmaname}{Lemma}
  \providecommand{\propositionname}{Proposition}
  \providecommand{\remarkname}{Remark}
  \providecommand{\factname}{Fact}
\providecommand{\theoremname}{Theorem}

\begin{document}
\twocolumn[
\icmltitle{Active Learning with Logged Data}



\begin{icmlauthorlist}
\icmlauthor{Songbai Yan}{aff}
\icmlauthor{Kamalika Chaudhuri}{aff}
\icmlauthor{Tara Javidi}{aff}
\end{icmlauthorlist}

\icmlaffiliation{aff}{University of California, San Diego}
\icmlcorrespondingauthor{Songbai Yan}{yansongbai@eng.ucsd.edu}

\icmlkeywords{Active Learning}

\vskip 0.3in
]



\printAffiliationsAndNotice{}  

\begin{abstract}
We consider active learning with logged data, where labeled examples are drawn conditioned on a predetermined logging policy, and the goal is to learn a classifier on the entire population, not just conditioned on the logging policy. Prior work addresses this problem either when only logged data is available, or purely in a controlled random experimentation setting where the logged data is ignored. In this work, we combine both approaches to provide an algorithm that uses logged data to bootstrap and inform experimentation, thus achieving the best of both worlds. Our work is inspired by a connection between controlled random experimentation and active learning, and modifies existing disagreement-based active learning algorithms to exploit logged data.
\end{abstract}

\section{Introduction}

We consider learning a classifier from logged data. Here, the learner has access to a logged labeled dataset that has been collected according to a known pre-determined policy, and his goal is to learn a classifier that predicts the labels accurately over the entire population, not just conditioned on the logging policy. 

This problem arises frequently in many natural settings. An example is predicting the efficacy of a treatment as a function of patient characteristics based on observed data. Doctors may assign the treatment to patients based on some predetermined rule; recording these patient outcomes produces a logged dataset where outcomes are observed conditioned on the doctors' assignment. A second example is recidivism prediction, where the goal is to predict whether a convict will re-offend. Judges use their own predefined policy to grant parole, and if parole is granted, then an outcome (reoffense or not) is observed. Thus the observed data records outcomes conditioned on the judges' parole policy, while the learner's goal is to learn a predictor over the entire population. 

A major challenge in learning from logged data is that the logging policy may leave large areas of the data distribution under-explored. Consequently, empirical risk minimization (ERM) on the logged data leads to classifiers that may be highly suboptimal on the population. When the logging policy is known, a second option is to use a {\em{weighted}} ERM, that reweighs each observed labeled data point to ensure that it reflects the underlying population. However, this may lead to sample inefficiency if the logging policy does not adequately explore essential regions of the population. A final approach, typically used in clinical trials, is controlled random experimentation -- essentially, ignore the logged data, and record outcomes for fresh examples drawn from the population. This approach is expensive due to the high cost of trials, and wasteful since it ignores the observed data. 

Motivated by these challenges, we propose active learning to combine logged data with a small amount of strategically chosen labeled data that can be used to correct the bias in the logging policy. This solution has the potential to achieve the best of both worlds by limiting experimentation to achieve higher sample efficiency, and by making the most of the logged data. Specifically, we assume that in addition to the logged data, the learner has some additional unlabeled data that he can selectively ask an annotator to label. The learner's goal is to learn a highly accurate classifier over the entire population by using a combination of the logged data and with as few label queries to the annotator as possible.

How can we utilize logged data for better active learning? This problem has not been studied to the best of our knowledge. A naive approach is to use the logged data to come up with a {\em{warm start}} and then do standard active learning. In this work, we show that we can do even better. In addition to the warm start, we show how to use multiple importance sampling estimators to utilize the logged data more efficiently. Additionally, we introduce a novel debiasing policy that selectively avoids label queries for those examples that are highly represented in the logged data. 

Combining these three approaches, we provide a new algorithm. We prove that our algorithm is statistically consistent, and has a lower label requirement than simple active learning that uses the logged data as a warm start. Finally, we evaluate our algorithm experimentally on various datasets and logging policies. Our experiments show that the performance of our method is either the best or close to the best for a variety of datasets and logging policies. This confirms that active learning to combine logged data with carefully chosen labeled data may indeed yield performance gains.

\section{Preliminaries}
\subsection{Problem Setup}
Instances are drawn from an instance space $\mathcal{X}$ and a label space
$\mathcal{Y}=\{0,1\}$. There is an underlying data distribution $D$ over $\mathcal{X}\times\mathcal{Y}$ that describes the population. There is a hypothesis space $\mathcal{H}\subset\mathcal{Y}^{\mathcal{X}}$. For simplicity, we assume $\mathcal{H}$ is a finite set, but our results can be generalized to VC-classes by standard arguments \cite{VC71}.

The learning algorithm has access to two sources of data: logged data, and online data. The logged data are generated from $m$ examples $\{(X_{t},Y_{t})\}_{t=1}^{m}$ drawn i.i.d. from $D$, and a logging policy $Q_0: \mathcal{X} \rightarrow [0,1]$ that determines the probability of observing the label. For each example $(X_{t},Y_{t})$ ($1\leq t \leq m$), an independent Bernoulli random variable $Z_t$ is drawn with expectation $Q_0(X_t)$, and then the label $Y_t$ is revealed to the learning algorithm if $Z_t=1$\footnote{Note that this generating process implies the standard unconfoundedness assumption in the counterfactual inference literature: $\Pr(Y_t,Z_t\mid X_t)=\Pr(Y_t\mid X_t)\Pr(Z_t\mid X_t)$, that is, given the instance $X_t$, its label $Y_t$ is conditionally independent with the action $Z_t$ (whether the label is observed).}. We call $T_{0}=\{(X_{t},Y_{t},Z_{t})\}_{t=1}^{m}$ the logged dataset. From the algorithm's perspective, we assume it knows the logging policy $Q_0$, and only observes instances $\{X_t\}_{t=1}^{m}$, decisions of the policy $\{Z_t\}_{t=1}^{m}$, and revealed labels $\{Y_{t}\mid Z_t=1\}_{t=1}^{m}$. 


The online data are generated as follows. Suppose there is a stream of another $n$ examples $\{(X_{t},Y_{t})\}_{t=m+1}^{m+n}$ drawn i.i.d. from distribution $D$. At time $t$ ($m<t\leq m+n$), the algorithm uses its query policy to compute a bit $Z_t \in \{0,1\}$, and then the label $Y_t$ is revealed to the algorithm if $Z_t=1$. The computation of $Z_t$ may in general be randomized, and is based on the observed logged data $T_0$, observed instances $\{X_i\}_{i=m+1}^{t}$, previous decisions$\{Z_i\}_{i=m+1}^{t-1}$, and observed labels $\{Y_i\mid Z_i=1\}_{i=m+1}^{t-1}$.

The goal of the algorithm is to learn a classifier $h\in\mathcal{H}$ from observed logged data and online data. Fixing $D$, $Q_0$, $m$, $n$, the performance measures are: (1) the error rate $l(h):=\Pr_D(h(X)\neq Y)$ of the output classifier, and (2) the number of label queries on the online data. Note that the error rate is over the entire population $D$ instead of conditioned on the logging policy, and that we assume the logged data $T_0$ come at no cost. 
In this work, we are interested in the situation where $n$ is about the same as or less than $m$.

\subsection{Background on Disagreement-Based Active Learning}
Our algorithm is based on Disagreement-Based Active Learning (DBAL) which has rigorous theoretical guarantees and can be implemented practically (see \cite{H14} for a survey, and \cite{HY15,HAHLS15} for some recent developments). DBAL iteratively maintains a candidate set of classifiers that contains the optimal classifier $h^{\star}:=\arg\min_{h\in\mathcal{H}}l(h)$ with high probability. At the $k$-th iteration, the candidate set $V_k$ is constructed as all classifiers which have low estimated error on examples observed up to round $k$. Based on $V_k$, the algorithm constructs a disagreement set $D_k$ to be a set of instances on which there are at least two classifiers in $V_k$ that predict different labels. Then the algorithm draws a set $T_k$ of unlabeled examples, where the size of $T_k$ is a parameter of the algorithm. For each instance $X\in T_k$, if it falls into the disagreement region $D_k$, then the algorithm queries for its label; otherwise, observing that all classifiers in $V_k$ have the same prediction on $X$, its label is not queried. The queried labels are then used to update future candidate sets. 

\subsection{Background on Error Estimators}\label{sec:estimator}
Most learning algorithms, including DBAL, require estimating the error rate of a classifier. A good error estimator should be unbiased and of low variance. When instances are observed with different probabilities, a commonly used error estimator is the standard importance sampling estimator that reweighs each observed labeled example according to the inverse probability of observing it.

Consider a simplified setting where the logged dataset $T_0 = (X_i, Y_i, Z_i)_{i=1}^{m}$ and $\Pr(Z_i=1\mid X_i)=Q_0(X_i)$. On the online dataset $T_1 = (X_i, Y_i, Z_i)_{i=m+1}^{m+n}$,  the algorithm uses a fixed query policy $Q_1$ to determine whether to query for labels, that is, $\Pr(Z_i=1\mid X_i)=Q_1(X_i)$ for $m<i\leq m+n$. Let $S = T_0 \cup T_1$.

In this setting, the standard importance sampling (IS) error estimator for a classifier $h$ is:
\begin{align}
l_{\text{IS}}(h,S) := & \frac{1}{m+n}\sum_{i=1}^{m}\frac{\One\{h(X_i)\neq Y_i\}Z_i}{Q_0(X_i)}\nonumber\\
 & +\frac{1}{m+n}\sum_{i=m+1}^{m+n}\frac{\One\{h(X_i)\neq Y_i\}Z_i}{Q_1(X_i)}.\label{eq:IS}
\end{align}

$l_{\text{IS}}$ is unbiased, and its variance is proportional to $\sup_{i=0,1;x\in\mathcal{X}}\frac{1}{Q_i(x)}$. Although the learning algorithm can choose its query policy $Q_1$ to avoid $Q_1(X_i)$ to be too small for $i>m$, $Q_0$ is the logging policy that cannot be changed. When $Q_0(X_i)$ is small for some $i\leq m$, the estimator in (\ref{eq:IS}) have a high variance such that it may be even better to just ignore the logged dataset $T_0$.

An alternative is the multiple importance sampling (MIS) estimator with balanced heuristic \cite{VG95}:
\begin{equation} \label{eq:MIS}
l_{\text{MIS}}(h,S):=\sum_{i=1}^{m+n}\frac{\One\{h(X_i)\neq Y_i\}Z_i}{mQ_0(X_i)+nQ_1(X_i)}.
\end{equation}

It can be proved that $l_{\text{MIS}}(h,S)$ is indeed an unbiased estimator for $l(h)$.  Moreover, as proved in \cite{OZ00,ABSJ17}, (\ref{eq:MIS}) always has a lower variance than both (\ref{eq:IS}) and the standard importance sampling estimator that ignores the logged data.

In this paper, we use multiple importance sampling estimators, and write $l_{\text{MIS}}(h,S)$ as $l(h,S)$.

\paragraph{Additional Notations}
In this paper, unless otherwise specified, all probabilities and expectations are over the distribution $D$, and we drop $D$ from subscripts henceforth. 

Let $\rho(h_{1},h_{2}):=\Pr(h_{1}(X)\neq h_{2}(X))$ be the disagreement mass between $h_1$ and $h_2$, and $\rho_{S}(h_{1},h_{2}):=\frac{1}{N}\sum_{i=1}^{N}\One\{h_{1}(x_{i})\neq h_{2}(x_{i})\}$ for $S=\{x_{1},x_{2},\dots,x_{N}\}\subset\mathcal{X}$ be the empirical disagreement mass between $h_1$ and $h_2$ on $S$.

For any $h\in\mathcal{H}$, $r>0$, define $B(h,r):=\{h'\in\mathcal{H}\mid\rho(h,h')\leq r\}$ to be $r$-ball around $h$.
For any $V\subseteq\mathcal{H}$, define the disagreement region $\text{DIS}(V):=\{x\in\mathcal{X}\mid\exists h_{1}\neq h_{2}\in V\text{ s.t. }h_{1}(x)\neq h_{2}(x)\}$.

\section{Algorithm}
\subsection{Main Ideas}\label{subsec:main-ideas}
Our algorithm employs the disagreement-based active learning framework, but modifies the main DBAL algorithm in three key ways.
\subsubsection*{Key Idea 1: Warm-Start}
Our algorithm applies a straightforward way of making use of the logged data $T_0$ inside the DBAL framework: to set the initial candidate set $V_0$ to be the set of classifiers that have a low empirical error on $T_0$.

\subsubsection*{Key Idea 2: Multiple Importance Sampling}
Our algorithm uses multiple importance sampling estimators instead of standard importance sampling estimators. As noted in the previous section, in our setting, multiple importance sampling estimators are unbiased and have lower variance, which results in a better performance guarantee. 

We remark that the main purpose of using multiple importance sampling estimators here is to control the variance due to the predetermined logging policy. In the classical active learning setting without logged data, standard importance sampling can give satisfactory performance guarantees \cite{BDL09,BHLZ10,HAHLS15}.

\subsubsection*{Key Idea 3: A Debiasing Query Strategy}
The logging policy $Q_0$ introduces bias into the logged data: some examples may be underrepresented since $Q_0$ chooses to reveal their labels with lower probability. Our algorithm employs a debiasing query strategy to neutralize this effect. For any instance $x$ in the online data, the algorithm would query for its label with a lower probability if $Q_{0}(x)$ is relatively large.

It is clear that a lower query probability leads to fewer label queries. Moreover, we claim that our debiasing strategy, though queries for less labels, does not deteriorate our theoretical guarantee on the error rate of the final output classifier. To see this, we note that we can establish a concentration bound for multiple importance sampling estimators that with probability at least $1-\delta$, for all $h\in \mathcal{H}$,
\begin{align}
l(h)-l(h^{\star}) \leq &  2(l(h,S)-l(h^{\star},S))\nonumber\\
+ & \gamma_{1}\sup_{x\in\mathcal{X}}\frac{\One\{h(x)\neq h^{\star}(x)\}\log\frac{|\mathcal{H}|}{\delta}}{mQ_{0}(x)+nQ_1(x)}\nonumber\\
+ & \gamma_{1}\sqrt{\sup_{x\in\mathcal{X}}\frac{\One\{h(x)\neq h^{\star}(x)\}\log\frac{|\mathcal{H}|}{\delta}}{mQ_{0}(x)+nQ_1(x)}l(h^{\star})}\label{eq:main-concentration}
\end{align}
where $m,n$ are sizes of logged data and online data respectively, $Q_0$ and $Q_1$ are query policy during the logging phase and the online phase respectively, and $\gamma_1$ is an absolute constant (see Corollary~\ref{cor:gen} in Appendix for proof).

This concentration bound implies that for any $x\in\mathcal{X}$, if $Q_0(x)$ is large, we can set $Q_1(x)$ to be relatively small (as long as $mQ_0(x)+nQ_1(x) \geq \inf_{x'}mQ_0(x')+nQ_1(x')$) while achieving the same concentration bound. Consequently, the upper bound on the final error rate that we can establish from this concentration bound would not be impacted by the debiasing querying strategy.

One technical difficulty of applying both multiple importance sampling and the debiasing strategy to the DBAL framework is adaptivity. Applying both methods requires that the query policy and consequently the importance weights in the error estimator are updated with observed examples in each iteration. In this case, the summands of the error estimator are not independent, and the estimator becomes an adaptive multiple importance sampling estimator whose convergence property is still an open problem \cite{CMMR12}.

To circumvent this convergence issue and establish rigorous theoretical guarantees, in each iteration, we compute the error estimator from a fresh sample set. In particular, we partition the logged data and the online data stream into disjoint subsets, and we use one logged subset and one online subset for each iteration.

\subsection{Details of the Algorithm}
\begin{algorithm}
\begin{algorithmic}[1]
\STATE{Input: confidence $\delta$, size of online data $n$, logging policy $Q_0$, logged data $T_0$.}
\STATE{$K \gets \lceil\log{n}\rceil$.}
\STATE{$\tilde{S}_0 \gets T_0^{(0)}$; $V_0 \gets \mathcal{H}$; $D_0 \gets \mathcal{X}$; $\xi_0 \gets \inf_{x\in \mathcal{X}}Q_0(x)$.}
\FOR{$k=0, \dots, K-1$}
	\STATE{Define $\delta_k \gets \frac{\delta}{(k+1)(k+2)}$; $\sigma(k, \delta) \gets \frac{\log|\mathcal{H}|/\delta}{m_k\xi_k+n_k}$;
    $\Delta_k(h,h')\gets\gamma_0(\sigma(k, \frac{\delta_k}{2}) + \sqrt{\sigma(k, \frac{\delta_k}{2}) \rho_{\tilde{S}_{k}}(h,h')})$.}
    \STATE \(\triangleright\) {$\gamma_0$ is an absolute constant defined in Lemma~\ref{lem:h_star_in}.}
	\STATE{$\hat{h}_k \gets \arg\min_{h\in V_k} l(h, \tilde{S}_k)$.}
	\STATE{Define the candidate set $$V_{k+1} \gets \{ h\in V_k \mid l(h,\tilde{S}_k) \leq l(\hat{h}_k,\tilde{S}_k)+\Delta_k(h,\hat{h}_k)\}$$ and its disagreement region $D_{k+1} \gets \text{DIS}(V_{k+1})$.}
    \STATE{Define $\xi_{k+1}\gets\inf_{x\in D_{k+1}}Q_0(x)$, and $Q_{k+1}(x)\gets\One\{Q_0(x)\leq \xi_{k+1} + 1/\alpha\}$.} 
	\STATE{Draw $n_{k+1}$ samples $\{(X_t,Y_t)\}_{t=m+n_1\cdots+n_k+1}^{m+n_1+\cdots+n_{k+1}}$, and present $\{X_t\}_{t=m+n_1+\cdots+n_k+1}^{m+n_1+\cdots+n_{k+1}}$ to the algorithm.}
	\FOR{$t=m+n_1+\cdots+n_k+1 \text{ to } m+n_1+\cdots+n_{k+1}$}
		\STATE{$Z_t \gets Q_{k+1}(X_t)$.}	
		\IF{$Z_t=1$}
		\STATE{If $X_t \in D_{k+1}$, query for label: $\tilde{Y}_t\gets Y_t$; otherwise infer $\tilde{Y}_t \gets \hat{h}_k(X_t)$.}
		\ENDIF
	\ENDFOR	
	\STATE{$\tilde{T}_{k+1} \gets \{X_t, \tilde{Y}_t, Z_t\}_{t=m+n_1+\cdots+n_k+1}^{m+n_1+\cdots+n_{k+1}}$.}
    \STATE{$\tilde{S}_{k+1} \gets T_0^{(k+1)}\cup \tilde{T}_{k+1}$.}
\ENDFOR
\STATE{Output $\hat{h}=\arg\min_{h\in V_{K}} l(h, \tilde{S}_{K})$.}
\end{algorithmic}
\caption{\label{alg:main}Acitve learning with logged data}
\end{algorithm}
The Algorithm is shown as Algorithm~\ref{alg:main}. Algorithm~\ref{alg:main} runs in $K$ iterations where $K=\lceil\log{n}\rceil$ (recall $n$ is the size of the online data stream). For simplicity, we assume $n=2^K-1$. 

As noted in the previous subsection, we require the algorithm to use a disjoint sample set for each iteration. Thus, we partition the data as follows. The online data stream is partitioned into $K$ parts $T_1,\cdots, T_K$ of sizes $n_1=2^0, \cdots, n_K=2^{K-1}$. We define $n_0=0$ for completeness. The logged data $T_0$ is partitioned into $K+1$ parts $T_0^{(0)},\cdots, T_0^{(K)}$ of sizes $m_0=m/3, m_1=\alpha n_1, m_2=\alpha n_2, \cdots, m_K=\alpha n_K$ (where $\alpha=2m/3n$ and we assume $\alpha\geq1$ is an integer for simplicity. $m_0$ can take other values as long as it is a constant factor of $m$). The algorithm uses $T_0^{(0)}$ to construct an initial candidate set, and uses $S_k := T_0^{(k)} \cup T_k$ in iteration $k$.

Algorithm~\ref{alg:main} uses the disagreement-based active learning framework. At iteration $k$ ($k=0,\cdots,K-1$), it first constructs a candidate set $V_{k+1}$ which is the set of classifiers whose training error (using the multiple importance sampling estimator) on $T_0^{(k)} \cup \tilde{T}_k$ is small, and its disagreement region $D_{k+1}$. At the end of the $k$-th iteration, it receives the $(k+1)$-th part of the online data stream $\{X_i\}_{i=m+n_1\cdots+n_k+1}^{m+n_1\cdots+n_{k+1}}$ from which it can query for labels. It only queries for labels inside the disagreement region $D_{k+1}$. For any example $X$ outside the disagreement region, Algorithm~\ref{alg:main} infers its label $\tilde{Y}=\hat{h}_k(X)$. Throughout this paper, we denote by $T_k$, $S_k$ the set of examples with original labels, and by $\tilde{T}_k$, $\tilde{S_k}$ the set of examples with inferred labels. The algorithm only observes $\tilde{T}_k$ and $\tilde{S_k}$.

Algorithm~\ref{alg:main} uses aforementioned debiasing query strategy, which leads to fewer label queries than the standard disagreement-based algorithms. To simplify our analysis, we round the query probability $Q_k(x)$ to be 0 or 1.

\section{Analysis}
\subsection{Consistency}
We first introduce some additional quantities.

Define $h^\star:=\min_{h\in\mathcal{H}}l(h)$ to be the best classifier in $\mathcal{H}$, and $\nu:=l(h^\star)$ to be its error rate. Let $\gamma_2$ to be an absolute constant to be specified in Lemma~\ref{lem:dis-radius} in Appendix. 

We introduce some definitions that will be used to upper-bound the size of the disagreement sets in our algorithm. Let $\text{DIS}_0:=\mathcal{X}$. Recall $K=\lceil\log n\rceil$. For $k=1,\dots,K$, let $\zeta_k := \sup_{x\in\text{DIS}_{k-1}}\frac{\log(2|\mathcal{H}|/\delta_k)}{m_{k-1}Q_0(x)+n_{k-1}}$, $\epsilon_{k}:=\gamma_{2}\zeta_k+\gamma_{2}\sqrt{\zeta_kl(h^{\star})}$, $\text{DIS}_{k}:=\text{DIS}(B(h^{\star},2\nu+\epsilon_{k}))$. Let $\zeta := \sup_{x\in\text{DIS}_1}\frac1{\alpha Q_0(x)+1}$.

The following theorem gives statistical consistency of our algorithm.
\begin{thm}
\label{thm:Convergence} There is an absolute constant
$c_{0}$ such that for any $\delta>0$, with probability at least
$1-\delta$,
\begin{align*}
l(\hat{h})\leq & l(h^{\star})+c_{0}\sup_{x\in\text{DIS}_{K}}\frac{\log\frac{K|\mathcal{H}|}{\delta}}{m Q_{0}(x)+n} \\
& +c_{0}\sqrt{\sup_{x\in\text{DIS}_{K}}\frac{\log\frac{K|\mathcal{H}|}{\delta}}{m Q_{0}(x)+n}l(h^{\star})}.
\end{align*}
\end{thm}
\subsection{Label Complexity}
We first introduce the adjusted disagreement coefficient, which characterizes the rate of decrease of the query region as the candidate set shrinks.
\begin{defn}
For any measurable set $A\subseteq \mathcal{X}$, define $S(A, \alpha)$ to be
\[
\bigcup_{A'\subseteq A} \left(A'\cap\left\{ x:Q_{0}(x)\leq\inf_{x\in A'}Q_{0}(x)+\frac{1}{\alpha}\right\} \right).
\]
For any $r_{0}\geq2\nu$, $\alpha\geq1$, define the adjusted disagreement coefficient $\tilde{\theta}(r_{0},\alpha)$ to be
\[
\sup_{r>r_{0}} \frac{1}{r}\Pr(S(\text{DIS}(B(h^{\star},r)), \alpha )).
\]
\end{defn}

The adjusted disagreement coefficient is a generalization of the standard disagreement coefficient \citep{H07} which has been widely used for analyzing active learning algorithms. The standard disagreement coefficient $\theta(r)$ can be written as $\theta(r) = \tilde{\theta}(r,1)$, and clearly $\theta(r)\geq \tilde{\theta}(r,\alpha)$ for all $\alpha\geq1$.

We can upper-bound the number of labels queried by our algorithm using the adjusted disagreement coefficient. (Recall that we only count labels queried during the online phase, and that $\alpha=2m/3n\geq1$)
\begin{thm}
\label{thm:Label-Complexity}There is an absolute
constant $c_{1}$ such that for any $\delta>0$, with probability
at least $1-\delta$, the number of labels queried by Algorithm~\ref{alg:main}
is at most: 

\begin{align*}
c_1\tilde{\theta}(2\nu+\epsilon_K,\alpha)( & n\nu+\zeta\log n\log\frac{|\mathcal{H}|\log n}{\delta}\\
& +\log n\sqrt{n\nu\zeta\log\frac{|\mathcal{H}|\log n}{\delta}}).
\end{align*}
\end{thm}

\subsection{Remarks}
As a sanity check, note that when $Q_{0}(x)\equiv1$ (i.e., all labels in the logged data are shown), our results reduce to the classical bounds for disagreement-based active learning with a warm-start.

Next, we compare the theoretical guarantees of our algorithm with some alternatives. We fix the target error rate to be $\nu+\epsilon$, assume we are given $m$ logged data, and compare upper bounds on the number of labels required in the online phase to achieve the target error rate. Recall $\xi_0=\inf_{x\in\mathcal{X}}Q_0(x)$. Define $\tilde{\xi}_K:=\inf_{x\in\text{DIS}_K}Q_0(x)$, $\tilde{\theta}:=\tilde{\theta}(2\nu,\alpha)$, $\theta:=\theta(2\nu)$.

From Theorem~\ref{thm:Convergence} and \ref{thm:Label-Complexity} and some algebra, our algorithm requires $\tilde{O}\left(\nu\tilde{\theta}\cdot(\frac{\nu+\epsilon}{\epsilon^2}\log\frac{|\mathcal{H}|}{\delta}-m\tilde{\xi}_K)\right)$ labels.

The first alternative is passive learning that requests all labels for $\{X_t\}_{t=m+1}^{m+n}$ and finds an empirical risk minimizer using both logged data and online data. If standard importance sampling is used, the upper bound is $\tilde{O}\left(\frac{1}{\xi_0}(\frac{\nu+\epsilon}{\epsilon^2}\log\frac{|\mathcal{H}|}{\delta}-m\xi_0)\right)$.  If multiple importance sampling is used, the upper bound is $\tilde{O}\left(\frac{\nu+\epsilon}{\epsilon^2}\log\frac{|\mathcal{H}|}{\delta}-m\tilde{\xi}_K\right)$. Both bounds are worse than ours since $\nu\tilde{\theta}\leq1$ and $\xi_0\leq\tilde{\xi}_K\leq1$.

A second alternative is standard disagreement-based active learning with naive warm-start where the logged data is only used to construct an initial candidate set. For standard importance sampling, the upper bound is $\tilde{O}\left(\frac{\nu\theta}{\xi_0}(\frac{\nu+\epsilon}{\epsilon^2}\log\frac{|\mathcal{H}|}{\delta}-m\xi_0)\right)$. For multiple importance sampling (i.e., out algorithm without the debiasing step), the upper bound is $\tilde{O}\left(\nu\theta\cdot(\frac{\nu+\epsilon}{\epsilon^2}\log\frac{|\mathcal{H}|}{\delta}-m\tilde{\xi}_K)\right)$. Both bounds are worse than ours since $\nu\tilde{\theta}\leq\nu\theta$ and $\xi_0\leq\tilde{\xi}_K\leq1$.

A third alternative is to merely use past policy to label data -- that is, query on $x$ with probability $Q_0(x)$ in the online phase. The upper bound here is $\tilde{O}\left(\frac{\E[Q_0(X)]}{\xi_0}(\frac{\nu+\epsilon}{\epsilon^2}\log\frac{|\mathcal{H}|}{\delta}-m\xi_0)\right)$. This is worse than ours since $\xi_0\leq\E[Q_0(X)]$ and $\xi_0\leq\tilde{\xi}_K\leq1$.

\section{Experiments}
We now empirically validate our theoretical results by comparing our algorithm with a few alternatives on several datasets and logging policies.  In particular, we confirm that the test error of our classifier drops faster than several alternatives as the expected number of label queries increases. Furthermore, we investigate the effectiveness of two key components of our algorithm: multiple importance sampling and the debiasing query strategy.

\subsection{Methodology}
\subsubsection{Algorithms and Implementations}
To the best of our knowledge, no algorithms with theoretical guarantees have been proposed in the literature. We consider the overall performance of our algorithm against two natural baselines: standard passive learning (\textsc{\passive}) and the disagreement-based active learning algorithm with warm start (\textsc{\dbalw}). To understand the contribution of multiple importance sampling and the debiasing query strategy, we also compare the results with the disagreement-based active learning with warm start that uses multiple importance sampling (\textsc{\dbalwm}). We do not compare with the standard disagreement-based active learning that ignores the logged data since the contribution of warm start is clear: it always results in a smaller initial candidate set, and thus leads to less label queries.

Precisely, the algorithms we implement are:
\begin{itemize}
\item \textsc{\passive}: A passive learning algorithm that queries labels for all examples in the online sequence and uses the standard importance sampling estimator to combine logged data and online data.
\item \textsc{\dbalw}: A disagreement-based active learning algorithm that uses the standard importance sampling estimator, and constructs the initial candidate set with logged data. This algorithm only uses only our first key idea -- warm start.
\item \textsc{\dbalwm}: A disagreement-based active learning algorithm that uses the multiple importance sampling estimator, and constructs the initial candidate set with logged data. This algorithm uses our first and second key ideas, but not the debiasing query strategy. In other words, this method sets $Q_k\equiv 1$ in Algorithm~\ref{alg:main}.
\item \textsc{\idbal}: The method proposed in this paper: improved disagreement-based active learning algorithm with warm start that uses the multiple importance sampling estimator and the debiasing query strategy.
\end{itemize}

Our implementation of above algorithms follows Vowpal Wabbit \cite{vw}. Details can be found in Appendix.

\subsubsection{Data}
Due to lack of public datasets for learning with logged data, we convert datasets for standard binary classification into our setting. Specifically, we first randomly select 80\% of the whole dataset as training data and  the remaining 20\% is test data. We randomly select 50\% of the training set as logged data, and the remaining 50\% is online data. We then run an artificial logging policy (to be specified later) on the logged data to determine whether each label should be revealed to the learning algorithm or not.

Experiments are conducted on synthetic data and 11 datasets from UCI datasets \cite{L13} and LIBSVM datasets \cite{CL11}. The synthetic data is generated as follows: we generate 6000 30-dimensional points uniformly from hypercube $[-1,1]^{30}$, and labels are assigned by a random linear classifier and then flipped with probability 0.1 independently.

We use the following four logging policies:

\begin{itemize}
\item \textsc{Identical:} Each label is revealed with probability 0.005.
\item \textsc{Uniform:} We first assign each instance in the instance space to three groups with (approximately) equal probability. Then the labels in each group are revealed with probability 0.005, 0.05, and 0.5 respectively.
\item \textsc{Uncertainty:} We first train a coarse linear classifier using 10\% of the data. Then, for an instance at distance $r$ to the decision boundary, we reveal its label with probability $\exp(-cr^2)$ where $c$ is some constant. This policy is intended to simulate uncertainty sampling used in active learning. 
\item \textsc{Certainty:} We first train a coarse linear classifier using 10\% of the data. Then, for an instance at distance $r$ to the decision boundary, we reveal its label with probability $cr^2$ where $c$ is some constant. This policy is intended to simulate a scenario where an action (i.e. querying for labels in our setting) is taken only if the current model is certain about its consequence.
\end{itemize}

\subsubsection{Metrics and Parameter Tuning}
The experiments are conducted as follows. For a fixed policy, for each dataset $d$, we repeat the following process  10 times. At time $k$, we first randomly generate a simulated logged dataset, an online dataset, and a test dataset as stated above. Then for $i=1, 2, \cdots$, we set the horizon of the online data stream $a_i = 10\times 2^i$ (in other words, we only allow the algorithm to use first $a_i$ examples in the online dataset), and run algorithm $A$ with parameter set $p$  (to be specified later) using the logged dataset and first $a_i$ examples in the online dataset. We record $n(d,k,i,A,p)$ to be the number of label queries, and $e(d,k,i,A,p)$ to be the test error of the learned linear classifier.

Let $\bar{n}(d,i,A,p)=\frac{1}{10}\sum_{k}n(d,k,i,A,p)$, $\bar{e}(d,i,A,p)=\frac{1}{10}\sum_{k}e(d,k,i,A,p)$. To evaluate the overall performance of algorithm $A$ with parameter set $p$, we use the following area under the curve metric (see also \cite{HAHLS15}):
\begin{align*}
\text{AUC}(d,A,p) = \sum_i & \frac{\bar{e}(d,i,A,p)+\bar{e}(d,i+1,A,p)}{2} \\
& \cdot (\bar{n}(d,i+1,A,p)-\bar{n}(d,i,A,p)).
\end{align*}
A small value of AUC means that the test error decays fast as the number of label queries increases.

The parameter set $p$ consists of two parameters:
\begin{itemize}
\item Model capacity $C$ (see also item~\ref{item:C} in Appendix~\ref{subsec:appendix-implement}). In our theoretical analysis there is a term $C:=O(\log \frac{\mathcal{H}}{\delta})$ in the bounds, which is known to be loose in practice \cite{H10}. Therefore, in experiments, we treat $C$ as a parameter to tune. We try $C$ in $\{0.01 \times 2^k\mid k=0, 2, 4, \dots, 18\}$
\item Learning rate $\eta$ (see also item~\ref{item:eta} in Appendix~\ref{subsec:appendix-implement}). We use online gradient descent with stepsize $\sqrt{\frac{\eta}{t+\eta}}$ . We try $\eta$ in $\{0.0001\times 2^k\mid k=0, 2, 4, \dots, 18\}$.
\end{itemize}

For each policy, we report $\text{AUC}(d,A)=\min_p \text{AUC}(d,A,p)$, the AUC under the parameter set that minimizes AUC for dataset $d$ and algorithm $A$.

\subsection{Results and Discussion}
\begin{table*}[tb]
\begin{minipage}{.5\linewidth}
\centering
\caption{AUC under Identical policy}\label{tab:auc-identical}
\begin{tabular}{lllll}
\toprule
Dataset & \passive & \dbalw & \dbalwm & \idbal \\
\midrule
synthetic & 121.77 & 123.61 & 111.16 & \textbf{106.66} \\
letter & 4.40 & 3.65 & 3.82 & \textbf{3.48} \\
skin & 27.53 & 27.29 & 21.48 & \textbf{21.44} \\
magic & 109.46 & 101.77 & 89.95 & \textbf{83.82} \\
covtype & 228.04 & 209.56 & \textbf{208.82} & 220.27 \\
mushrooms & 19.22 & 25.29 & \textbf{18.54} & 23.67 \\
phishing & 78.49 & 73.40 & \textbf{70.54} & 71.68 \\
splice & 65.97 & 67.54 & 65.73 & \textbf{65.66} \\
svmguide1 & 59.36 & 55.78 & \textbf{46.79} & 48.04 \\
a5a & 53.34 & \textbf{50.8} & 51.10 & 51.21 \\
cod-rna & 175.88 & 176.42 & 167.42 & \textbf{164.96} \\
german & 65.76 & 68.68 & \textbf{59.31} & 61.54 \\
\bottomrule
\end{tabular}
\end{minipage}%
\begin{minipage}{.5\linewidth}
\centering
\caption{AUC under Uniform policy}\label{tab:auc-uniform}
\begin{tabular}{lllll}
\toprule
Dataset & \passive & \dbalw & \dbalwm & \idbal \\
\midrule
synthetic & 113.49 & 106.24 & 92.67 & \textbf{88.38} \\
letter & 1.68 & \textbf{1.29} & 1.45 & 1.59 \\
skin & 23.76 & 21.42 & 20.67 & \textbf{19.58} \\
magic & 53.63 & 51.43 & 51.78 & \textbf{50.19} \\
covtype & \textbf{262.34} & 287.40 & 274.81 & 263.82 \\
mushrooms & 7.31 & 6.81 & \textbf{6.51} & 6.90 \\
phishing & 42.53 & 39.56 & 39.19 & \textbf{37.02} \\
splice & 88.61 & 89.61 & 90.98 & \textbf{87.75} \\
svmguide1 & 110.06 & 105.63 & 98.41 & \textbf{96.46} \\
a5a & \textbf{46.96} & 48.79 & 49.50 & 47.60 \\
cod-rna & 63.39 & 63.30 & 66.32 & \textbf{58.48} \\
german & 63.60 & 55.87 & 56.22 & \textbf{55.79} \\
\bottomrule
\end{tabular}
\end{minipage} 
\end{table*}
\begin{table*}[tb]
\begin{minipage}{.5\linewidth}
\centering
\caption{AUC under Uncertainty policy}\label{tab:auc-uncertainty}
\begin{tabular}{lllll}
\toprule
Dataset & \passive & \dbalw & \dbalwm & \idbal \\
\midrule
synthetic & 117.86 & 113.34 & 100.82 & \textbf{99.1} \\
letter & \textbf{0.65} & 0.70 & 0.71 & 1.07 \\
skin & 20.19 & 21.91 & \textbf{18.89} & 19.10 \\
magic & 106.48 & 101.90 & 99.44 & \textbf{90.05} \\
covtype & 272.48 & 274.53 & 271.37 & \textbf{251.56} \\
mushrooms & 4.93 & 4.64 & 3.77 & \textbf{2.87} \\
phishing & 52.96 & 48.62 & \textbf{46.55} & 46.59 \\
splice & 62.94 & 63.49 & 60.00 & \textbf{58.56} \\
svmguide1 & 117.59 & 111.58 & \textbf{98.88} & 100.44 \\
a5a & 70.97 & 72.15 & \textbf{65.37} & 69.54 \\
cod-rna & 60.12 & 61.66 & 64.48 & \textbf{53.38} \\
german & 62.64 & 58.87 & 56.91 & \textbf{56.67} \\
\bottomrule
\end{tabular}
\end{minipage} 
\begin{minipage}{.5\linewidth}
\centering
\caption{AUC under Certainty policy}\label{tab:auc-certainty}
\begin{tabular}{lllll}
\toprule
Dataset & \passive & \dbalw & \dbalwm & \idbal \\
\midrule
synthetic & 114.86 & 111.02 & 92.39 & \textbf{88.82} \\
letter & 2.02 & \textbf{1.43} & 2.46 & 1.87 \\
skin & 22.89 & \textbf{17.92} & 18.17 & 18.11 \\
magic & 231.64 & 225.59 & 205.95 & \textbf{202.29} \\
covtype & 235.68 & 240.86 & 228.94 & \textbf{216.57} \\
mushrooms & 16.53 & 14.62 & 17.97 & \textbf{11.65} \\
phishing & 34.70 & 37.83 & 35.28 & \textbf{33.73} \\
splice & 125.32 & 129.46 & 122.74 & \textbf{122.26} \\
svmguide1 & 94.77 & 91.99 & 92.57 & \textbf{84.86} \\
a5a & \textbf{119.51} & 132.27 & 138.48 & 125.53 \\
cod-rna & 98.39 & 98.87 & 90.76 & \textbf{90.2} \\
german & 63.47 & \textbf{58.05} & 61.16 & 59.12 \\
\bottomrule
\end{tabular}
\end{minipage} 
\end{table*}
We report the AUCs for each algorithm under each policy and each dataset in Tables~\ref{tab:auc-identical} to \ref{tab:auc-certainty}. The test error curves can be found in Appendix. 

\paragraph{Overall Performance} The results confirm that the test error of the classifier output by our algorithm (\textsc{\idbal}) drops faster than the  baselines \textsc{passive} and \textsc{\dbalw}: as demonstrated in Tables~\ref{tab:auc-identical} to \ref{tab:auc-certainty}, \textsc{\idbal} achieves lower AUC than both \textsc{\passive} and \textsc{\dbalw} for a majority of datasets under all policies. We also see that \textsc{\idbal} performs better than or close to \textsc{\dbalwm} for all policies other than Identical. This confirms that among our two key novel ideas, using multiple importance sampling consistently results in a performance gain. Using the debiasing query strategy over multiple importance sampling also leads to performance gains, but these are less consistent.

\paragraph{The Effectiveness of Multiple Importance Sampling} As noted in Section~\ref{sec:estimator}, multiple importance sampling estimators have lower variance than standard importance sampling estimators, and thus can lead to a lower label complexity. This is verified in our experiments that \textsc{\dbalwm} (DBAL with multiple importance sampling estimators) has a lower AUC than \textsc{\dbalw} (DBAL with standard importance sampling estimator) on a majority of datasets under all policies.

\paragraph{The Effectiveness of the Debiasing Query Strategy} Under Identical policy, all labels in the logged data are revealed with equal probability. In this case, our algorithm \textsc{\idbal} queries all examples in the disagreement region as \textsc{\dbalwm} does. As shown in Table~\ref{tab:auc-identical},  \textsc{\idbal} and \textsc{\dbalwm} achieves the best AUC on similar number of datasets, and both methods outperform \textsc{\dbalw} over most datasets.

Under Uniform, Uncertainty, and Certainty policies, labels in the logged data are revealed with different probabilities. In this case, \textsc{\idbal}'s debiasing query strategy takes effect: it queries less frequently the instances that are well-represented in the logged data, and we show that this could lead to a lower label complexity theoretically. In our experiments, as shown in Tables~\ref{tab:auc-uniform} to \ref{tab:auc-certainty}, \textsc{\idbal} does indeed outperform \textsc{\dbalwm} on these policies empirically.


\section{Related Work}
Learning from logged observational data is a fundamental problem in machine learning with applications to causal inference \cite{SJS17}, information retrieval \cite{SLLK10, LCKG15, HLR16}, recommender systems \cite{LCLS10, SSSCJ16}, online learning \cite{AHKL+14, WA17}, and reinforcement learning \cite{Thomas2015, TTG15, Mandel16}. This problem is also closely related to covariate shift \cite{Z04, SKM07, BBCKPV10}. Two variants are widely studied -- first, when the logging policy is known, a problem known as learning from logged data \cite{LCKG15, TTG15,SJ15CRM,SJ15self}, and second, when this policy is unknown \cite{JSS16, AI16,Kallus17partitioning, SJS17}, a problem known as learning from observational data. Our work addresses the first problem. 

When the logging policy is {\em{unknown}}, the direct method \cite{DLL11} finds a classifier using observed data. This method, however, is vulnerable to selection bias \cite{HLR16,JSS16}. Existing de-biasing procedures include~\cite{AI16,Kallus17partitioning}, which proposes a tree-based method to partition the data space, and \cite{JSS16,SJS17}, which proposes to use deep neural networks to learn a good representation for both the logged and population data. 


When the logging policy is {\em{known}}, we can learn a classifier by optimizing a loss function that is an unbiased estimator of the expected error rate. Even in this case, however, estimating the expected error rate of a classifier is not completely straightforward and has been one of the central problems in contextual bandit \citep{WA17}, off-policy evaluation \citep{JL16}, and other related fields. The most common solution is to use importance sampling according to the inverse propensity scores \cite{RR83}. This method is unbiased when propensity scores are accurate, but may have high variance when some propensity scores are close to zero. To resolve this, \cite{BPQC+13, SLLK10,SJ15CRM} propose to truncate the inverse propensity score, \cite{SJ15self} proposes to use normalized importance sampling, and \cite{JL16, DLL11,TB16, WA17} propose doubly robust estimators. Recently, \cite{TTG15} and \cite{ABSJ17} suggest adjusting the importance weights according to data to further reduce the variance. We use the multiple importance sampling estimator (which have also been recently studied in \cite{ABSJ17} for policy evaluation), and we prove this estimator concentrates around the true expected loss tightly.

Most existing work on learning with logged data falls into the passive learning paradigm, that is, they first collect the observational data and then train a classifier. In this work, we allow for active learning, that is, the algorithm could adaptively collect some labeled data. It has been shown in the active learning literature that adaptively selecting data to label can achieve high accuracy at low labeling cost \cite{BBL09,BHLZ10,H14,ZC14,HAHLS15}.  \citet{KAHHL17} study active learning with bandit feedback and give a disagreement-based learning algorithm.

To the best of our knowledge, there is no prior work with theoretical guarantees that combines passive and active learning with a logged observational dataset. \citet{BDL09} consider active learning with warm-start where the algorithm is presented with a labeled dataset prior to active learning, but the labeled dataset is not observational: it is assumed to be drawn from the same distribution for the entire population, while in our work, we assume the logged dataset is in general drawn from a different distribution by a logging policy.

\section{Conclusion and Future Work}
We consider active learning with logged data. The logged data are collected by a predetermined logging policy while the learner's goal is to learn a classifier over the entire population. We propose a new disagreement-based active learning algorithm that makes use of warm start, multiple importance sampling, and a debiasing query strategy. We show that theoretically our algorithm achieves better label complexity than alternative methods. Our theoretical results are further validated by empirical experiments on different datasets and logging policies.

This work can be extended in several ways. First, the derivation and analysis of the debiasing strategy are based on a variant of the concentration inequality (\ref{eq:main-concentration}) in subsection~\ref{subsec:main-ideas}. The inequality relates the generalization error with the best error rate $l(h^\star)$, but has a looser variance term than some existing bounds (for example \citep{CMM10}).
A more refined analysis on the concentration of weighted estimators could better characterize the performance of the proposed algorithm, and might also improve the debiasing strategy. Second, due to the dependency of multiple importance sampling, in Algorithm~\ref{alg:main}, the candidate set $V_{k+1}$ is constructed with only the $k$-th segment of data $\tilde{S}_k$ instead of all data collected so far $\cup_{i=0}^{k}\tilde{S}_i$. One future direction is to investigate how to utilize all collected data while provably controlling the variance of the weighted estimator. Finally, it would be interesting to investigate how to perform active learning from logged observational data without knowing the logging policy.

\paragraph*{Acknowledgements} We thank NSF under CCF 1719133 for support. We thank Chris Meek, Adith Swaminathan, and Chicheng Zhang for helpful discussions. We also thank anonymous reviewers for constructive comments.

\bibliographystyle{icml2018}
\bibliography{counterfactual,lpactive}
\newpage
\onecolumn
\appendix
\section{Preliminaries}
\subsection{Summary of Key Notations}
\paragraph{Data Partitions} $T_k=\{(X_t,Y_t,Z_t)\}_{t=m+n_1+\cdots+n_{k-1}+1}^{t=m+n_1+\cdots+n_k}$ ($1\leq k\leq K$) is the online data collected in $k$-th iteration of size $n_k=2^{k-1}$. $n=n_1+\cdots+n_K$, $\alpha=2m/3n$. We define $n_0=0$. $T_0=\{(X_t,Y_t,Z_t)\}_{t=1}^{t=m}$ is the logged data and is partitioned into $K+1$ parts $T_0^{(0)}, \cdots, T_0^{(K)}$ of sizes $m_0=m/3, m_1=\alpha n_1, m_2=\alpha n_2, \cdots, m_K = \alpha n_K$. $S_k=T_0^{(k)}\cup T_k$. 

Recall that $\tilde{S}_k$ and $\tilde{T}_k$ contain inferred labels while $S_k$ and $T_k$ are sets of examples with original labels. The algorithm only observes $\tilde{S}_k$ and $\tilde{T}_k$.

For $(X,Z)\in T_k$ ($0\leq k\leq K)$, $Q_k(X) = \Pr(Z=1\mid X)$.

\paragraph{Disagreement Regions} The candidate set $V_{k}$ and its disagreement region $D_{k}$ are defined in Algorithm~\ref{alg:main}. $\hat{h}_{k}=\arg\min_{h\in V_{k}}l(h,\tilde{S}_{k})$. $\nu=l(h^\star)$.

$B(h,r):=\{h'\in\mathcal{H}\mid\rho(h,h')\leq r\}$, $\text{DIS}(V):=\{x\in\mathcal{X}\mid\exists h_{1}\neq h_{2}\in V\text{ s.t. }h_{1}(x)\neq h_{2}(x)\}$. $S(A, \alpha)=\bigcup_{A'\subseteq A} \left(A'\cap\left\{ x:Q_{0}(x)\leq\inf_{x\in A'}Q_{0}(x)+\frac{1}{\alpha}\right\} \right)$. $\tilde{\theta}(r_{0},\alpha)=\sup_{r>r_{0}} \frac{1}{r}\Pr(S(\text{DIS}(B(h^{\star},r)), \alpha ))$.

$\text{DIS}_0=\mathcal{X}$. For $k=1,\dots,K$, $\epsilon_{k}=\gamma_{2}\sup_{x\in\text{DIS}_{k-1}}\frac{\log(2|\mathcal{H}|/\delta_k)}{m_{k-1}Q_0(x)+n_{k-1}}+\gamma_{2}\sqrt{\sup_{x\in\text{DIS}_{k-1}}\frac{\log(2|\mathcal{H}|/\delta_k)}{m_{k-1}Q_0(x)+n_{k-1}}l(h^{\star})}$, $\text{DIS}_{k}=\text{DIS}(B(h^{\star},2\nu+\epsilon_{k}))$.

\paragraph{Other Notations} $\rho(h_1,h_2)=\Pr(h_1(X)\neq h_2(X))$, $\rho_S(h_1,h_2)=\frac{1}{|S|}\sum_{X\in S}\One\{h_1(X)\neq h_2(X)\}$.

For $k\geq0$, $\sigma(k,\delta)=\sup_{x\in D_{k}}\frac{\log(|\mathcal{H}|/\delta)}{m_k Q_{0}(x)+ n_k}$, $\delta_k=\frac{\delta}{(k+1)(k+2)}$. $\xi_k = \inf_{x\in D_k}Q_0(x)$. $\zeta = \sup_{x\in\text{DIS}_1}\frac1{\alpha Q_0(x)+1}$. 

\subsection{Elementary Facts}
\begin{prop}
\label{prop:quad-ineq}Suppose $a,c\geq0$,$b\in\R$. If $a\leq b+\sqrt{ca}$,
then $a\leq2b+c$.
\end{prop}
\begin{proof}
Since $a\leq b+\sqrt{ca}$, $\sqrt{a}\leq\frac{\sqrt{c}+\sqrt{c+4b}}{2}\leq\sqrt{\frac{c+c+4b}{2}}=\sqrt{c+2b}$
where the second inequality follows from the Root-Mean Square-Arithmetic
Mean inequality. Thus, $a\leq2b+c$.
\end{proof}
\subsection{Facts on Disagreement Regions and Candidate Sets}
\begin{lem}
\label{lem:Q}For any $k=0,\dots,K$, any $x\in\mathcal{X}$, any
$h_{1},h_{2}\in V_{k}$, $\frac{\One\{h_{1}(x)\neq h_{2}(x)\}}{m_k Q_{0}(X)+ n_k Q_{k}(X)}\leq\sup_{x'}\frac{\One\{x'\in D_{k}\}}{m_k Q_{0}(x')+n_k}$.
\end{lem}
\begin{proof}
The $k=0$ case is obvious since $D_0=\mathcal{X}$ and $n_0=0$.

For $k>0$, since $\text{DIS}(V_k)=D_k$, $\One\{h_{1}(x)\neq h_{2}(x)\} \leq \One\{x\in D_{k}\}$, and thus $\frac{\One\{h_{1}(x)\neq h_{2}(x)\}}{m_k Q_{0}(X)+ n_k Q_{k}(X)} \leq \frac{\One\{x\in D_{k}\}}{m_k Q_{0}(X)+ n_k Q_{k}(X)}$.

For any $x$, if $Q_0(x)\leq \xi_k+1/\alpha$, then $Q_k(x)=1$, so $\frac{\One\{x\in D_{k}\}}{m_k Q_{0}(X)+ n_k Q_{k}(X)} = \frac{\One\{x\in D_{k}\}}{m_k Q_{0}(x)+n_k} \leq \sup_{x'}\frac{\One\{x'\in D_{k}\}}{m_k Q_{0}(x')+n_k}$.

If $Q_0(x)> \xi_k+1/\alpha$, then $Q_k(x)=0$, so $\frac{\One\{x\in D_{k}\}}{m_k Q_{0}(X)+ n_k Q_{k}(X)} = \frac{\One\{x\in D_{k}\}}{m_{k} Q_{0}(x)} \leq \frac{\One\{x\in D_{k}\}}{m_k \xi_k+n_k} \leq \sup_{x'}\frac{\One\{x'\in D_{k}\}}{m_k Q_{0}(x')+n_k}$ where the first inequality follows from the fact that $Q_0(x)> \xi_k+1/\alpha$ implies $m_k Q_0(x)> m_k\xi_k+n_k$
\end{proof}

\begin{lem}
\label{lem:l-diff-S-S_tilde}For any $k=0,\dots,K$, if $h_{1},h_{2}\in V_{k}$,
then $l(h_{1},S_{k})-l(h_{2},S_{k})=l(h_{1},\tilde{S}_{k})-l(h_{2},\tilde{S}_{k})$. 
\end{lem}
\begin{proof}
For any $(X_t, Y_t, Z_t) \in S_t$ that $Z_t=1$, if $X_t\in\text{DIS}(V_k)$, then $Y_t=\tilde{Y}_t$, so $\One\{h_1(X_t)\neq Y_t\} - \One\{h_2(X_t)\neq Y_t\} = \One\{h_1(X_t)\neq \tilde{Y}_t\} - \One\{h_2(X_t)\neq \tilde{Y}_t\}$. If $X_t\notin\text{DIS}(V_k)$, then $h_1(X_t)=h_2(X_t)$, so $\One\{h_1(X_t)\neq Y_t\} - \One\{h_2(X_t)\neq Y_t\} = \One\{h_1(X_t)\neq \tilde{Y}_t\} - \One\{h_2(X_t)\neq \tilde{Y}_t\} = 0$.
\end{proof}

The following lemma is immediate from definition.
\begin{lem}
\label{lem:dis-coefficient}For any $r\geq2\nu$, any $\alpha\geq1$, $\Pr(S(\text{DIS}(B(h^{\star},r)),\alpha))\leq r\tilde{\theta}(r,\alpha)$.
\end{lem}
%

\subsection{Facts on Multiple Importance Sampling Estimators}
We recall that $\{(X_t,Y_t)\}_{t=1}^{n_0+n}$ is an i.i.d. sequence. Moreover, the following fact is immediate by our construction that $S_0,\cdots,S_K$ are disjoint and that $Q_k$ is determined by $S_0,\cdots,S_{k-1}$.
\begin{fact}
\label{fact:independence}For any $0\leq k \leq K$, conditioned on $Q_k$, examples in $S_k$ are independent, and examples in $T_k$ are i.i.d.. Besides, for any $0<k\leq K$, $Q_k$, $T_0^{(k)},\dots,T_0^{(K)}$ are independent.
\end{fact}

Unless otherwise specified, all probabilities and expectations are over the random draw of all random variables (including $S_0, \cdots, S_K$, $Q_1, \cdots, Q_K$).

The following lemma shows multiple importance estimators are unbiased.
\begin{lem}
\label{lem:mis-unbiased}For any $h\in\mathcal{H}$, any $0\leq k\leq K$, $\E[l(h,S_{k})]=l(h)$.
\end{lem}

The above lemma is immediate from the following lemma.
\begin{lem}
\label{lem:cond-mis-unbiased}For any $h\in\mathcal{H}$, any $0\leq k\leq K$, $\E[l(h,S_{k})\mid Q_k]=l(h)$.
\end{lem}
\begin{proof}
The $k=0$ case is obvious since $S_0=T_0^{(0)}$ is an i.i.d. sequence and $l(h,S_{k})$ reduces to a standard importance sampling estimator. We only show proof for $k>0$. 

Recall that $S_{k}=T_0^{(k)}\cup T_k$, and that $T_0^{(k)}$ and $T_k$ are two i.i.d. sequences conditioned $Q_k$. We denote the conditional distributions of $T_0^{(k)}$ and $T_k$ given $Q_k$ by $P_0$ and $P_k$ respectively. We have
\begin{eqnarray*}
\E[l(h,S_{k})\mid Q_k] & = & \E\left[\sum_{(X,Y,Z)\in T_0^{(k)}}\frac{\One\{h(X)\neq Y\}Z}{m_k Q_{0}(X)+ n_k Q_{k}(X)}\mid Q_k\right] + \E\left[\sum_{(X,Y,Z)\in T_k}\frac{\One\{h(X)\neq Y\}Z}{m_k Q_{0}(X)+ n_k Q_{k}(X)}\mid Q_k\right]\\
 & = & m_k\E_{P_{0}}\left[\frac{\One\{h(X)\neq Y\}Z}{m_k Q_{0}(X)+ n_k Q_{k}(X)}\mid Q_k\right]+n_{k}\E_{P_{k}}\left[\frac{\One\{h(X)\neq Y\}Z}{m_k Q_{0}(X)+ n_k Q_{k}(X)}\mid Q_k\right]
\end{eqnarray*}
where the second equality follows since $T_0^{(k)}$ and $T_k$ are two i.i.d. sequences given $Q_k$ with sizes $m_k$ and $n_k$ respectively.

Now, 
\begin{eqnarray*}
\E_{P_{0}}\left[\frac{\One\{h(X)\neq Y\}Z}{m_k Q_{0}(X)+ n_k Q_{k}(X)}\mid Q_k\right] & = & \E_{P_0}\left[\E_{P_{0}}\left[\frac{\One\{h(X)\neq Y\}Z}{m_k Q_{0}(X)+ n_k Q_{k}(X)}\mid X, Q_k\right]\mid Q_k\right] \\
& = & \E_{P_{0}}\left[\E_{P_{0}}\left[\frac{\One\{h(X)\neq Y\}Q_{0}(X)}{m_k Q_{0}(X)+ n_k Q_{k}(X)}\mid X, Q_k\right]\mid Q_k\right] \\
& = & \E_{P_{0}}\left[\frac{\One\{h(X)\neq Y\}Q_{0}(X)}{m_k Q_{0}(X)+ n_k Q_{k}(X)}\mid Q_k\right]
\end{eqnarray*}
where the second equality uses the definition $\Pr_{P_0}(Z\mid X)=Q_0(X)$ and the fact that $T_0^{(k)}$ and $Q_k$ are independent.

Similarly, we have $\E_{P_{k}}\left[\frac{\One\{h(X)\neq Y\}Z}{m_k Q_{0}(X)+ n_k Q_{k}(X)}\mid Q_k\right] = \E_{P_{k}}\left[\frac{\One\{h(X)\neq Y\}Q_{k}(X)}{m_k Q_{0}(X)+ n_k Q_{k}(X)}\mid Q_k\right]$.

Therefore,
\begin{eqnarray*}
\lefteqn{m_{k}\E_{P_{0}}\left[\frac{\One\{h(X)\neq Y\}Z}{m_k Q_{0}(X)+ n_k Q_{k}(X)}\mid Q_k\right]+n_{k}\E_{P_{k}}\left[\frac{\One\{h(X)\neq Y\}Z}{m_k Q_{0}(X)+ n_k Q_{k}(X)}\mid Q_k\right]} \\
& = & m_k\E_{P_{0}}\left[\frac{\One\{h(X)\neq Y\}Q_{0}(X)}{m_k Q_{0}(X)+ n_k Q_{k}(X)}\mid Q_k\right] + n_k\E_{P_{k}}\left[\frac{\One\{h(X)\neq Y\}Q_{k}(X)}{m_k Q_{0}(X)+ n_k Q_{k}(X)}\mid Q_k\right] \\
& = & \E_{P_{0}}\left[\One\{h(X)\neq Y\}\frac{m_{k}Q_{0}(X)+n_{k}Q_{k}(X)}{m_k Q_{0}(X)+ n_k Q_{k}(X)}\mid Q_k\right]\\
& = & \E_D\left[\One\{h(X)\neq Y\}\right]=l(h)
\end{eqnarray*}
where the second equality uses the fact that distribution of $(X,Y)$ according to $P_0$ is the same as that according to $P_k$, and the third equality follows by algebra and Fact~\ref{fact:independence} that $Q_k$ is independent with $T_0^{(k)}$.
\end{proof}

The following lemma will be used to upper-bound the variance of the multiple importance sampling estimator.

\begin{lem}
\label{lem:var-mis}For any $h_1,h_2\in\mathcal{H}$, any $0\leq k\leq K$, 
\[
\E\left[\sum_{(X,Y,Z)\in S_k}\left(\frac{\One\{h_{1}(X)\neq h_{2}(X)\}Z}{m_k Q_{0}(X)+n_kQ_{k}(X)}\right)^2\mid Q_k\right] \leq \rho(h_1,h_2)\sup_{x\in\mathcal{X}}\frac{\One\{h_{1}(x)\neq h_{2}(x)\}}{m_k Q_0(x)+n_kQ_k(x)}.
\]
\end{lem}
\begin{proof}
We only show proof for $k>0$. The $k=0$ case can be proved similarly.

We denote the conditional distributions of $T_0^{(k)}$ and $T_k$ given $Q_k$ by $P_0$ and $P_k$ respectively. Now, similar to the proof of Lemma~\ref{lem:cond-mis-unbiased}, we have
\begin{align*}
\lefteqn{\E\left[\sum_{(X,Y,Z)\in S_k}\left(\frac{\One\{h_{1}(X)\neq h_{2}(X)\}Z}{m_k Q_{0}(X)+n_kQ_{k}(X)}\right)^2\mid Q_k\right]}\\
= & \sum_{(X,Y,Z)\in S_k}\E\left[\frac{\One\{h_{1}(X)\neq h_{2}(X)\}Z}{\left(m_k Q_{0}(X)+ n_k Q_{k}(X)\right)^{2}}\mid Q_k\right]\\
= & m_{k}\E_{P_{0}}\left[\frac{\One\{h_{1}(X)\neq h_{2}(X)\}Z}{\left(m_k Q_{0}(X)+ n_k Q_{k}(X)\right)^{2}}\mid Q_k\right]+n_{k}\E_{P_{k}}\left[\frac{\One\{h_{1}(X)\neq h_{2}(X)\}Z}{\left(m_k Q_{0}(X)+ n_k Q_{k}(X)\right)^{2}}\mid Q_k\right]\\
= & m_{k}\E_{P_{0}}\left[\frac{\One\{h_{1}(X)\neq h_{2}(X)\}Q_0(X)}{\left(m_k Q_{0}(X)+ n_k Q_{k}(X)\right)^{2}}\mid Q_k\right]+n_{k}\E_{P_{k}}\left[\frac{\One\{h_{1}(X)\neq h_{2}(X)\}Q_k(X)}{\left(m_k Q_{0}(X)+ n_k Q_{k}(X)\right)^{2}}\mid Q_k\right]\\
= & \E_{P_0}\left[\One\{h_{1}(X)\neq h_{2}(X)\}\frac{m_k Q_0(X) + n_k Q_k(X)}{(m_k Q_{0}(X)+ n_k Q_{k}(X))^2}\mid Q_k\right]\\
= & \E_{P_0}\left[\frac{\One\{h_{1}(X)\neq h_{2}(X)\}}{m_k Q_{0}(X)+ n_k Q_{k}(X)}\mid Q_k\right]\\
\leq & \E_{P_0}\left[\One\{h_{1}(X)\neq h_{2}(X)\}\mid Q_k\right]\sup_{x\in\mathcal{X}}\frac{\One\{h_{1}(x)\neq h_{2}(x)\}}{m_k Q_0(x)+ n_k Q_k(x)} \\
= & \rho(h_1,h_2)\sup_{x\in\mathcal{X}}\frac{\One\{h_{1}(x)\neq h_{2}(x)\}}{m_k Q_0(x)+ n_k Q_k(x)}.
\end{align*}
\end{proof}
\section{Deviation Bounds}
In this section, we demonstrate deviation bounds for our error estimators on $S_k$. Again, unless otherwise specified, all probabilities and expectations in this section are over the random draw of all random variables, that is, $S_0, \cdots, S_K$, $Q_1, \cdots, Q_K$.

We use following Bernstein-style concentration bound:
\begin{fact}
\label{fact:bernstein}Suppose $X_{1},\dots,X_{n}$ are independent
random variables. For any $i=1,\dots,n$, $|X_{i}|\leq1$, $\E X_{i}=0$,
$\E X_{i}^{2}\leq\sigma_{i}^{2}$. Then with probability at
least $1-\delta$,
\[
\left|\sum_{i=1}^{n}X_{i}\right|\leq\frac{2}{3}\log\frac{2}{\delta}+\sqrt{2\sum_{i=1}^{n}\sigma_{i}^{2}\log\frac{2}{\delta}}.
\]
\end{fact}

\begin{thm}
\label{thm:gen}For any $k=0,\dots,K$, any $\delta>0$, with probability
at least $1-\delta$, for all $h_{1},h_{2}\in\mathcal{H},$
the following statement holds:
\begin{align}
\left|\left(l(h_{1},S_{k})-l(h_{2},S_{k})\right)-\left(l(h_{1})-l(h_{2})\right)\right| & \leq2\sup_{x\in\mathcal{X}}\frac{\One\{h_{1}(x)\neq h_{2}(x)\}\frac{2\log\frac{4|\mathcal{H}|}{\delta}}{3}}{m_k Q_{0}(x)+ n_k Q_{k}(x)}+\sqrt{2\sup_{x\in\mathcal{X}}\frac{\One\{h_{1}(x)\neq h_{2}(x)\}\log\frac{4|\mathcal{H}|}{\delta}}{m_k Q_{0}(x)+ n_k Q_{k}(x)}\rho(h_{1},h_{2})}.\label{eq:thm-gen-diff-rho}
\end{align}
\end{thm}
\begin{proof}
We show proof for $k>0$. The $k=0$ case can be proved similarly. When $k>0$, it suffices to show that for any $k=1,\dots, K$, $\delta>0$, conditioned on $Q_k$, with probability at least $1-\delta$, (\ref{eq:thm-gen-diff-rho}) holds for all $h_1,h_2\in\mathcal{H}$.

For any $k=1,\dots,K$, for any fixed $h_{1},h_{2}\in\mathcal{H}$, define $A:=\sup_{x\in\mathcal{X}}\frac{\One\{h_{1}(x)\neq h_{2}(x)\}}{m_k Q_{0}(x)+ n_k Q_{k}(x)}$. Let $N:=|S_k|$, 
$U_{t}:=\frac{\One\{h_{1}(X_{t})\neq Y_{t}\}Z_{t}}{m_k Q_{0}(X_{t})+n_k Q_{k}(X_{t})}-\frac{\One\{h_{2}(X_{t})\neq Y_{t}\}Z_{t}}{m_k Q_{0}(X_{t})+n_k Q_{k}(X_{t})}$,
$V_{t}:=(U_{t}-\E[U_{t}|Q_k])/2A$. 

Now, conditioned on $Q_k$, $\{V_{t}\}_{t=1}^{N}$ is an independent sequence by Fact~\ref{fact:independence}. $|V_{t}|\leq1$, and $\E[V_{t}|Q_k]=0$. Besides, we have

\begin{eqnarray*}
\sum_{t=1}^{N}\E[V_{t}^2|Q_k] & \leq & \frac{1}{4A^{2}}\sum_{t=1}^{N}\E[U_{t}^{2}|Q_k]\\
 & \leq & \frac{1}{4A^{2}}\sum_{t=1}^{N}\E\left(\frac{\One\{h_{1}(X_{t})\neq h_{2}(X_{t})\}Z_{t}}{m_k Q_{0}(X_{t})+n_k Q_{k}(X_{t})}\right)^{2}\\
 & \leq & \frac{\rho(h_1,h_2)}{4A}
\end{eqnarray*}
where the second inequality follows from $|U_{t}|\leq\frac{\One\{h_{1}(X_{t})\neq h_{2}(X_{t})\}Z_{t}}{m_k Q_{0}(X_{t})+n_k Q_{k}(X_{t})}$, and the third inequality follows from Lemma~\ref{lem:var-mis}.

Applying Bernstein's inequality (Fact~\ref{fact:bernstein}) to $\{V_t\}$, conditioned on $Q_k$, we have with probability at least
$1-\delta$, 
\[
\left|\sum_{t=1}^{m}V_{t}\right|\leq\frac{2}{3}\log\frac{2}{\delta}+\sqrt{\frac{\rho(h_{1},h_{2})}{2A}\log\frac{2}{\delta}}.
\]

Note that $\sum_{t=1}^{m}U_{t}=l(h_{1}, S_k)-l(h_{2}, S_k)$, and $\sum_{t=1}^{m}\E[U_{t}\mid Q_k]=l(h_{1})-l(h_{2})$ by Lemma~\ref{lem:cond-mis-unbiased}, so $\sum_{t=1}^{m}V_{t}=\frac{1}{2A}(l(h_{1},S_{k})-l(h_{2},S_{k})-l(h_{1})+l(h_{2}))$.
(\ref{eq:thm-gen-diff-rho}) follows by algebra and a union bound
over $\mathcal{H}$.
\end{proof}
\begin{thm}
\label{thm:gen-rho}For any $k=0,\dots,K$, any $\delta>0$, with probability
at least $1-\delta$, for all
$h_{1},h_{2}\in\mathcal{H},$ the following statements hold simultaneously:
\begin{align}
\rho_{S_{k}}(h_{1},h_{2}) & \leq2\rho(h_{1},h_{2})+\frac{10}{3}\sup_{x\in\mathcal{X}}\frac{\One\{h_{1}(x)\neq h_{2}(x)\}\log\frac{4|\mathcal{H}|}{\delta}}{m_k Q_{0}(x)+ n_k Q_{k}(x)};\label{eq:thm-gen-rho_S}\\
\rho(h_{1},h_{2}) & \leq2\rho_{S_{k}}(h_{1},h_{2})+\frac{7}{6}\sup_{x\in\mathcal{X}}\frac{\One\{h_{1}(x)\neq h_{2}(x)\}\log\frac{4|\mathcal{H}|}{\delta}}{m_k Q_{0}(x)+ n_k Q_{k}(x)}.\label{eq:thm-gen-rho}
\end{align}
\end{thm}
\begin{proof}
Let $N=|S_{k}|$. Note that for any $h_{1},h_{2}\in\mathcal{H}$,
$\rho_{S_{k}}(h_{1},h_{2})=\frac{1}{N}\sum_{t}\One\{h_{1}(X_{t})\neq h_{2}(X_{t})\}$,
which is the empirical average of an i.i.d. sequence. By Fact~\ref{fact:bernstein}
and a union bound over $\mathcal{H}$, with probability at least $1-\delta$, 

\[
\left|\rho(h_{1},h_{2})-\rho_{S_{k}}(h_{1},h_{2})\right|\leq\frac{2}{3N}\log\frac{4|\mathcal{H}|}{\delta}+\sqrt{\frac{2\rho(h_{1},h_{2})}{N}\log\frac{4|\mathcal{H}|}{\delta}}.
\]

On this event, by Proposition~\ref{prop:quad-ineq}, $\rho(h_{1},h_{2})\leq2\rho_{S_{k}}(h_{1},h_{2})+\frac{4}{3N}\log\frac{4|\mathcal{H}|}{\delta}+\frac{2}{N}\log\frac{4|\mathcal{H}|}{\delta}\leq2\rho_{S_{k}}(h_{1},h_{2})+\frac{10}{3N}\log\frac{4|\mathcal{H}|}{\delta}$.

Moreover, 
\begin{eqnarray*}
\rho_{S_{k}}(h_{1},h_{2}) & \leq & \rho(h_{1},h_{2})+\frac{2}{3N}\log\frac{4|\mathcal{H}|}{\delta}+\sqrt{\frac{2\rho(h_{1},h_{2})}{N}\log\frac{4|\mathcal{H}|}{\delta}}\\
 & \leq & \rho(h_{1},h_{2})+\frac{2}{3N}\log\frac{4|\mathcal{H}|}{\delta}+\frac{1}{2}(2\rho(h_{1},h_{2})+\frac{1}{N}\log\frac{4|\mathcal{H}|}{\delta})\\
 & \leq & 2\rho(h_{1},h_{2})+\frac{7}{6N}\log\frac{4|\mathcal{H}|}{\delta}
\end{eqnarray*}

where the second inequality uses the fact that $\forall a,b>0,\sqrt{ab}\leq\frac{a+b}{2}$.

The result follows by noting that $\forall x\in\mathcal{X}$, $N=|S_{k}|=m_k+n_k\geq m_k Q_{0}(x)+ n_k Q_{k}(x)$.
\end{proof}
\begin{cor}
\label{cor:gen}There are universal constants $\gamma_{0},\gamma_{1}>0$
such that for any $k=0,\dots,K$, any $\delta>0$, with probability at
least $1-\delta$, for all $h,h_{1},h_{2}\in\mathcal{H},$
the following statements hold simultaneously:
\begin{equation}
\left|\left(l(h_{1},S_{k})-l(h_{2},S_{k})\right)-\left(l(h_{1})-l(h_{2})\right)\right|\leq\gamma_{0}\sup_{x\in\mathcal{X}}\frac{\One\{h_{1}(x)\neq h_{2}(x)\}\log\frac{|\mathcal{H}|}{2\delta}}{m_k Q_{0}(x)+ n_k Q_{k}(x)}+\gamma_{0}\sqrt{\sup_{x\in\mathcal{X}}\frac{\One\{h_{1}(x)\neq h_{2}(x)\}\log\frac{|\mathcal{H}|}{2\delta}}{m_k Q_{0}(x)+ n_k Q_{k}(x)}\rho_{S}(h_{1},h_{2})};\label{eq:cor-gen-diff-rho_S}
\end{equation}
\begin{equation}
l(h)-l(h^{\star})\leq2(l(h,S_{k})-l(h^{\star},S_{k}))+\gamma_{1}\sup_{x\in\mathcal{X}}\frac{\One\{h(x)\neq h^{\star}(x)\}\log\frac{|\mathcal{H}|}{\delta}}{m_k Q_{0}(x)+ n_k Q_{k}(x)}+\gamma_{1}\sqrt{\sup_{x\in\mathcal{X}}\frac{\One\{h(x)\neq h^{\star}(x)\}\log\frac{|\mathcal{H}|}{\delta}}{m_k Q_{0}(x)+ n_k Q_{k}(x)}l(h^{\star})}.\label{eq:cor-gen-diff-h_star}
\end{equation}
\end{cor}
\begin{proof}
Let event $E$ be the event that (\ref{eq:thm-gen-diff-rho}) and
(\ref{eq:thm-gen-rho}) holds for all $h_{1},h_{2}\in\mathcal{H}$
with confidence $1-\frac{\delta}{2}$ respectively. Assume $E$ happens
(whose probability is at least $1-\delta$).

(\ref{eq:cor-gen-diff-rho_S}) is immediate from (\ref{eq:thm-gen-diff-rho})
and (\ref{eq:thm-gen-rho}).

For the proof of (\ref{eq:cor-gen-diff-h_star}), apply (\ref{eq:thm-gen-diff-rho})
to $h$ and $h^{\star}$, we get 
\[
l(h)-l(h^{\star})\leq l(h,S_{k})-l(h^{\star},S_{k})+2\sup_{x\in\mathcal{X}}\frac{\One\{h(x)\neq h^{\star}(x)\}\frac{2\log\frac{4|\mathcal{H}|}{\delta}}{3}}{m_k Q_{0}(x)+ n_k Q_{k}(x)}+\sqrt{2\sup_{x\in\mathcal{X}}\frac{\One\{h(x)\neq h^{\star}(x)\}\log\frac{4|\mathcal{H}|}{\delta}}{m_k Q_{0}(x)+ n_k Q_{k}(x)}\rho(h,h^{\star})}.
\]

By triangle inequality, $\rho(h,h^{\star})=\Pr_D(h(X)\neq h^{\star}(X))\leq\Pr_D(h(X)\neq Y)+\Pr_D(h^{\star}(X)\neq Y)=l(h)-l(h^{\star})+2l(h^{\star})$.
Therefore, we get
\begin{eqnarray*}
l(h)-l(h^{\star}) & \leq & l(h,S_{k})-l(h^{\star},S_{k})+2\sup_{x\in\mathcal{X}}\frac{\One\{h(x)\neq h^{\star}(x)\}\frac{2\log\frac{4|\mathcal{H}|}{\delta}}{3}}{m_k Q_{0}(x)+ n_k Q_{k}(x)}\\
& & +\sqrt{2\sup_{x\in\mathcal{X}}\frac{\One\{h(x)\neq h^{\star}(x))\}\log\frac{4|\mathcal{H}|}{\delta}}{m_k Q_{0}(x)+ n_k Q_{k}(x)}(l(h)-l(h^{\star})+2l(h^{\star}))}\\
 & \leq & l(h,S_{k})-l(h^{\star},S_{k})+\sqrt{2\sup_{x\in\mathcal{X}}\frac{\One\{h(x)\neq h^{\star}(x)\}\log\frac{4|\mathcal{H}|}{\delta}}{m_k Q_{0}(x)+ n_k Q_{k}(x)}(l(h)-l(h^{\star}))}\\
 &  & +2\sup_{x\in\mathcal{X}}\frac{\One\{h(x)\neq h^{\star}(x)\}\frac{2\log\frac{4|\mathcal{H}|}{\delta}}{3}}{m_k Q_{0}(x)+ n_k Q_{k}(x)}+\sqrt{4\sup_{x\in\mathcal{X}}\frac{\One\{h(x)\neq h^{\star}(x)\}\log\frac{4|\mathcal{H}|}{\delta}}{m_k Q_{0}(x)+ n_k Q_{k}(x)}l(h^{\star})}
\end{eqnarray*}
where the second inequality uses $\sqrt{a+b}\leq \sqrt{a}+\sqrt{b}$ for $a, b\geq 0$.

(\ref{eq:cor-gen-diff-h_star}) follows by applying Proposition~\ref{prop:quad-ineq}
to $l(h)-l(h^{\star})$.
\end{proof}

\section{Technical Lemmas}
For any $0\leq k\leq K$ and $\delta>0$, define event $\mathcal{E}_{k,\delta}$ to be the event that the conclusions of Theorem~\ref{thm:gen} and Theorem~\ref{thm:gen-rho} hold for $k$ with confidence $1-\delta/2$ respectively. We have $\Pr(\mathcal{E}_{k,\delta})\geq 1-\delta$, and that $\mathcal{E}_{k,\delta}$ implies inequalities (\ref{eq:thm-gen-diff-rho}) to (\ref{eq:cor-gen-diff-h_star}).

We first present a lemma which can be used to guarantee that $h^\star$ stays in candidate sets with high probability by induction..

\begin{lem}
\label{lem:h_star_in}For any $k=0,\dots K$, any $\delta>0$. On event $\mathcal{E}_{k,\delta}$, if $h^{\star}\in V_{k}$ then, 
\[
l(h^{\star},\tilde{S}_{k})\leq l(\hat{h}_{k},\tilde{S}_{k})+\gamma_{0}\sigma(k,\delta)+\gamma_{0}\sqrt{\sigma(k,\delta)\rho_{\tilde{S}_{k}}(\hat{h}_{k},h^\star)}.
\]
\end{lem}
\begin{proof}
\begin{align*}
 & l(h^{\star},\tilde{S}_{k})-l(\hat{h}_{k},\tilde{S}_{k})\\
= & l(h^{\star},S_{k})-l(\hat{h}_{k},S_{k})\\
\leq & \gamma_{0}\sup_{x}\frac{\One\{h^{\star}(x)\neq\hat{h}_{k}(x)\}\log\frac{|\mathcal{H}|}{\delta}}{m_k Q_{0}(x)+ n_k Q_{k}(x)}+\gamma_{0}\sqrt{\sup_{x}\frac{\One\{h^{\star}(x)\neq\hat{h}_{k}(x)\}\log\frac{|\mathcal{H}|}{\delta}}{m_k Q_{0}(x)+ n_k Q_{k}(x)}\rho_{S_{k}}(\hat{h}_{k},h^\star)}\\
\leq & \gamma_{0}\sigma(k,\delta)+\sqrt{\gamma_{0}\sigma(k,\delta)\rho_{\tilde{S}_{k}}(\hat{h}_{k},h)}.
\end{align*}

The equality follows from Lemma~\ref{lem:l-diff-S-S_tilde}. The
first inequality follows from (\ref{eq:cor-gen-diff-rho_S}) of Corollary~\ref{cor:gen}
and that $l(h^{\star})\leq l(\hat{h}_{k})$. The last inequality follows
from Lemma~\ref{lem:Q} and that $\rho_{\tilde{S}_{k}}(\hat{h}_{k},h^\star) = \rho_{S_{k}}(\hat{h}_{k},h^\star)$.
\end{proof}

Next, we present two lemmas to bound the probability mass of the disagreement region of candidate sets.
\begin{lem}
\label{lem:dis-radius} For any $k=0,\dots,K$, any $\delta>0$, let $V_{k+1}(\delta):=\{h\in V_{k}\mid l(h,\tilde{S}_{k})\leq l(\hat{h}_{k},\tilde{S}_{k})+\gamma_{0}\sigma(k,\delta)+\gamma_{0}\sqrt{\sigma(k,\delta)\rho_{\tilde{S}_{k}}(\hat{h}_{k},h)}\}$. Then there is an absolute constant $\gamma_{2}>1$ such that for any $0,\dots,K$, any $\delta>0$, on event $\mathcal{E}_{k,\delta}$, if $h^{\star}\in V_{k}$, then for all $h\in V_{k+1}(\delta)$, 
\[
l(h)-l(h^{\star})\leq\gamma_{2}\sigma(k,\delta)+\gamma_{2}\sqrt{\sigma(k,\delta)l(h^{\star})}.
\]
\end{lem}
\begin{proof}
For any $h\in V_{k+1}(\delta)$, we have 
\begin{align}
\MoveEqLeft l(h)-l(h^{\star})\nonumber \\
\leq & 2(l(h,S_{k})-l(h^{\star},S_{k}))+\gamma_{1}\sigma(k,\frac{\delta}{2})+\gamma_{1}\sqrt{\sigma(k,\frac{\delta}{2})l(h^{\star})}\nonumber \\
= & 2(l(h,\tilde{S}_{k})-l(h^{\star},\tilde{S}_{k}))+\gamma_{1}\sigma(k,\frac{\delta}{2})+\gamma_{1}\sqrt{\sigma(k,\frac{\delta}{2})l(h^{\star})}\nonumber \\
= & 2(l(h,\tilde{S}_{k})-l(\hat{h}_{k},\tilde{S}_{k})+l(\hat{h}_{k},\tilde{S}_{k})-l(h^{\star},\tilde{S}_{k}))+\gamma_{1}\sigma(k,\frac{\delta}{2})+\gamma_{1}\sqrt{\sigma(k,\frac{\delta}{2})l(h^{\star})}\nonumber \\
\leq & 2(l(h,\tilde{S}_{k})-l(\hat{h}_{k},\tilde{S}_{k}))+\gamma_{1}\sigma(k,\frac{\delta}{2})+\gamma_{1}\sqrt{\sigma(k,\frac{\delta}{2})l(h^{\star})}\nonumber \\
\leq & (2\gamma_{0}+\gamma_{1})\sigma(k,\frac{\delta}{2})+2\gamma_{0}\sqrt{\sigma(k,\frac{\delta}{2})\rho_{\tilde{S}_{k}}(h,\hat{h}_{k})}+\gamma_{1}\sqrt{\sigma(k,\frac{\delta}{2})l(h^{\star})}\nonumber \\
\leq & (2\gamma_{0}+\gamma_{1})\sigma(k,\frac{\delta}{2})+2\gamma_{0}\sqrt{\sigma(k,\frac{\delta}{2})(\rho_{S_{k}}(h,h^{\star})+\rho_{S_{k}}(\hat{h_{k}},h^{\star}))}+\gamma_{1}\sqrt{\sigma(k,\frac{\delta}{2})l(h^{\star})} \label{eq:lem-dis-radius-1}
\end{align}

where the first inequality follows from (\ref{eq:cor-gen-diff-h_star})
of Corollary~\ref{cor:gen} and Lemma~\ref{lem:Q}, the first equality
follows from Lemma~\ref{lem:l-diff-S-S_tilde}, the third inequality
follows from the definition of $V_{k}(\delta)$, and the last inequality follows from $\rho_{\tilde{S}_{k}}(h,\hat{h}_{k})=\rho_{S_{k}}(h,\hat{h}_{k})\leq\rho_{S_{k}}(h,h^{\star})+\rho_{S_{k}}(\hat{h_{k}},h^{\star})$.

As for $\rho_{S_{k}}(h,h^{\star})$, we have $\rho_{S_{k}}(h,h^{\star})\leq2\rho(h,h^{\star})+\frac{16}{3}\sigma(k,\frac{\delta}{8})\leq2(l(h)-l(h^{\star}))+4l(h^{\star})+\frac{16}{3}\sigma(k,\frac{\delta}{8})$
where the first inequality follows from (\ref{eq:thm-gen-rho_S}) of
Theorem~\ref{thm:gen-rho} and Lemma~\ref{lem:Q}, and the second
inequality follows from the triangle inequality. 

For $\rho_{S_{k}}(\hat{h}_{k},h^{\star})$, we have 
\begin{eqnarray*}
\rho_{S_{k}}(\hat{h}_{k},h^{\star}) & \leq & 2\rho(\hat{h}_{k},h^{\star})+\frac{16}{3}\sigma(k,\frac{\delta}{8})\\
 & \leq & 2(l(\hat{h}_{k})-l(h^{\star})+2l(h^{\star}))+\frac{16}{3}\sigma(k,\frac{\delta}{8})\\
 & \leq & 2(2(l(\hat{h}_{k},S_{k})-l(h^{\star},S_{k}))+\gamma_{1}\sigma(k,\frac{\delta}{2})+\gamma_{1}\sqrt{\sigma(k,\frac{\delta}{2})l(h^{\star})}+2l(h^{\star}))+\frac{16}{3}\sigma(k,\frac{\delta}{8})\\
 & \leq & (2\gamma_{1}+\frac{16}{3})\sigma(k,\frac{\delta}{8})+2\gamma_{1}\sqrt{\sigma(k,\frac{\delta}{2})l(h^{\star})}+4l(h^{\star})\\
 & \leq & (4+\gamma_{1})l(h^{\star})+(3\gamma_{1}+\frac{16}{3})\sigma(k,\frac{\delta}{8})
\end{eqnarray*}
 where the first inequality follows from (\ref{eq:thm-gen-rho_S}) of
Theorem~\ref{thm:gen-rho} and Lemma~\ref{lem:Q}, the second follows from the triangle inequality,
the third follows from (\ref{eq:cor-gen-diff-h_star}) of Theorem~\ref{cor:gen}
and Lemma~\ref{lem:Q}, the fourth follows from the definition of $\hat{h}_{k}$,
the last follows from the fact that $2\sqrt{ab}\leq a+b$ for $a,b\geq0$.

Continuing (\ref{eq:lem-dis-radius-1}) and using the fact that $\sqrt{a+b}\leq\sqrt{a}+\sqrt{b}$
for $a,b\geq0$, we have:
\[
l(h)-l(h^{\star})\leq(2\gamma_{0}+\gamma_{1}+2\gamma_{0}\sqrt{3\gamma_{1}+\frac{32}{3}})\sigma(k,\frac{\delta}{8})+(2\gamma_{0}\sqrt{8+\gamma_{1}}+\gamma_{1})\sqrt{\sigma(k,\frac{\delta}{8})l(h^{\star})}+2\sqrt{2}\gamma_{0}\sqrt{\sigma(k,\frac{\delta}{8})(l(h)-l(h^{\star}))}.
\]

The result follows by applying Proposition~\ref{prop:quad-ineq}
to $l(h)-l(h^{\star})$.
\end{proof}
\begin{lem}
\label{lem:Dk-DISk}On event $\bigcap_{k=0}^{K-1}\mathcal{E}_{k,\delta_k/2}$, for any $k=0,\dots K$, $D_{k}\subseteq\text{DIS}_{k}$.
\end{lem}
\begin{proof}
Recall that $\delta_{k}=\frac{\delta}{(k+1)(k+2)}$. On event $\bigcap_{k=0}^{K-1}\mathcal{E}_{k,\delta_k/2}$, $h^\star \in V_k$ for all $0\leq k\leq K$ by Lemma~\ref{lem:h_star_in} and induction.

The $k=0$ case is obvious since $D_0=\text{DIS}_0=\mathcal{X}$. Now, suppose $0\leq k<K$, and $D_k \subseteq \text{DIS}_k$. We have
\begin{eqnarray*}
D_{k+1} & \subseteq & \text{DIS}\left(\left\{h:l(h)\leq\nu+\gamma_{2}\left(\sigma(k,\delta_{k}/2)+\sqrt{\sigma(k,\delta_{k}/2)\nu}\right)\right\}\right)\\
 & \subseteq & \text{DIS}\left(B\left(h^{\star},2\nu+\gamma_{2}\left(\sigma(k,\delta_{k}/2)+\sqrt{\sigma(k,\delta_{k}/2)\nu}\right)\right)\right)
\end{eqnarray*}
where the first line follows from Lemma~\ref{lem:dis-radius} and the
definition of $D_{k}$, and the second line follows from triangle inequality
that $\Pr(h(X)\neq h^{\star}(X))\leq l(h)+l(h^{\star})$ (recall $\nu=l(h^{\star})$).

To prove $D_{k+1}\subseteq \text{DIS}_{k+1}$ it suffices to show $\gamma_{2}\left(\sigma(k,\delta_{k}/2)+\sqrt{\sigma(k,\delta_{k}/2)\nu}\right)\leq\epsilon_{k+1}$.

Note that $\sigma(k,\delta_{k}/2)=\sup_{x\in D_{k}}\frac{\log(2|\mathcal{H}|/\delta_{k})}{m_k Q_{0}(x)+n_k}\leq\sup_{x\in\text{DIS}_{k}}\frac{\log(2|\mathcal{H}|/\delta_{k})}{m_k Q_{0}(x)+n_k}$ since $D_{k}\subseteq \text{DIS}_{k}$. Consequently, $\gamma_2\left(\sigma(k,\delta_{k}/2)+\sqrt{\sigma(k,\delta_{k}/2)\nu}\right)\leq\epsilon_{k+1}$.
\end{proof}

\section{Proof of Consistency}
\begin{proof}
(of Theorem~\ref{thm:Convergence}) Define event
$\mathcal{E}^{(0)}:=\bigcap_{k=0}^{K}\mathcal{E}_{k,\delta_k/2}$. By a union bound, $\Pr(\mathcal{E}^{(0)})\geq 1-\delta$. On event $\mathcal{E}^{(0)}$, by induction and Lemma~\ref{lem:h_star_in}, for all $k=0,\dots,K$, $h^{\star}\in V_{k}$. Observe that $\hat{h}=\hat{h}_K\in V_{K+1}(\delta_{K}/2)$. Applying Lemma~\ref{lem:dis-radius} to $\hat{h}$, we have
\[
l(\hat{h})\leq l(h^{\star})+\gamma_2\left(\sup_{x\in D_K}\frac{\log(2|\mathcal{H}|/\delta_K)}{m_K Q_{0}(x)+ n_K}+\sqrt{\sup_{x\in D_K}\frac{\log(2|\mathcal{H}|/\delta_K)}{m_K Q_{0}(x)+ n_K}l(h^{\star})}\right).
\]
The result follows by noting
that $\sup_{x\in\mathcal{X}}\frac{\One\{x\in D_{K}\}}{m_K Q_{0}(x)+n_K}\leq\sup_{x\in\mathcal{X}}\frac{\One\{x\in\text{DIS}_{K}\}}{m_K Q_{0}(x)+n_K}$ by Lemma~\ref{lem:Dk-DISk}.
\end{proof}

\section{Proof of Label Complexity}
\begin{proof}
(of Theorem~\ref{thm:Label-Complexity}) Recall that $\xi_{k}=\inf_{x\in D_{k}}Q_{0}(x)$ and $\zeta = \sup_{x\in\text{DIS}_1}\frac1{\alpha Q_0(x)+1}$. 

Define event $\mathcal{E}^{(0)}:=\bigcap_{k=0}^{K}\mathcal{E}_{k,\delta_k/2}$. On this event, by induction and Lemma~\ref{lem:h_star_in}, for all $k=0,\dots,K$, $h^{\star}\in V_{k}$, and consequently by Lemma~\ref{lem:Dk-DISk}, $D_k\subseteq\text{DIS}_k$.

For any $k=0,\dots K-1$, let the number of label queries at iteration
$k$ to be $U_k:=\sum_{t=n_{0}+\cdots+n_{k}+1}^{n_{0}+\cdots+n_{k+1}}Z_{t}\One\{X_{t}\in D_{k+1}\}$. 
\begin{eqnarray*}
Z_{t}\One\{X_{t}\in D_{k+1}\} & = & \One\{X_{t}\in D_{k+1}\land Q_0(X_t)\leq \inf_{x\in D_{k+1}}Q_0(x)+\frac{1}{\alpha}\} \\
 & \leq & \One\{X_t \in S(D_{k+1},\alpha)\} \\
 & \leq & \One\{X_t \in S(\text{DIS}_{k+1},\alpha)\}.
\end{eqnarray*}

Thus, $U_k\leq \sum_{t=n_{0}+\cdots+n_{k}+1}^{n_{0}+\cdots+n_{k+1}}\One\{X_t \in S(\text{DIS}_{k+1},\alpha)\}$, where the RHS is a sum of i.i.d. Bernoulli($\Pr(S(\text{DIS}_{k+1},\alpha))$) random variables, so a Bernstein
inequality implies that on an event $\mathcal{E}^{(1,k)}$ of probability
at least $1-\delta_{k}/2$, $\sum_{t=n_{0}+\cdots+n_{k}+1}^{n_{0}+\cdots+n_{k+1}}\One\{X_t \in S(\text{DIS}_{k+1},\alpha)\}\leq2n_{k+1}\Pr(S(\text{DIS}_{k+1},\alpha))+2\log\frac{4}{\delta_{k}}$.

Therefore, it suffices to show that on event $\mathcal{E}^{(2)} := \cap_{k=0}^{K}(\mathcal{E}^{(1,k)}\cap \mathcal{E}_{k,\delta_k/2})$, for some absolute constant $c_1$,
\[
\sum_{k=0}^{K-1}n_{k+1}\Pr(S(\text{DIS}_{k+1},\alpha)) \leq c_1\tilde{\theta}(2\nu+\epsilon_K,\alpha)(n\nu+\zeta\log n\log\frac{|\mathcal{H}|\log n}{\delta}+\log n\sqrt{n\nu\zeta\log\frac{|\mathcal{H}|\log n}{\delta}}).
\]

Now, on event $\mathcal{E}^{(2)}$, for any $k<K$, $\Pr(S(\text{DIS}_{k+1},\alpha))= \Pr(S(\text{DIS}(B(h^\star, 2\nu+\epsilon_{k+1})), \alpha))\leq (2\nu+\epsilon_{k+1})\tilde{\theta}(2\nu+\epsilon_{k+1},\alpha)$ where the last inequality follows from Lemma~\ref{lem:dis-coefficient}.

Therefore, 
\begin{align*}
\MoveEqLeft \sum_{k=0}^{K-1}n_{k+1}\Pr(S(\text{DIS}_{k+1},\alpha))\\
\leq & n_1+\sum_{k=1}^{K-1}n_{k+1}(2\nu+\epsilon_{k+1})\tilde{\theta}(2\nu+\epsilon_{k+1},\alpha)\\
\leq & 1+\tilde{\theta}(2\nu+\epsilon_K,\alpha)(2n\nu+\sum_{k=1}^{K-1}n_{k+1}\epsilon_{k+1})\\
\leq & 1+\tilde{\theta}(2\nu+\epsilon_K,\alpha)\left(2n\nu+2\gamma_{2}\sum_{k=1}^{K-1}(\sup_{x\in\text{DIS}_{1}}\frac{\log\frac{|\mathcal{H}|}{\delta_k/2}}{(\alpha Q_{0}(x)+1)}+\sqrt{n_k\nu\sup_{x\in\text{DIS}_{1}}\frac{\log\frac{|\mathcal{H}|}{\delta_k/2}}{(\alpha Q_{0}(x)+1)}})\right)\\
\leq & 1+\tilde{\theta}(2\nu+\epsilon_K,\alpha)(2n\nu+2\gamma_2\zeta\log n\log\frac{|\mathcal{H}|(\log n)^2}{\delta}+2\gamma_2\log n\sqrt{n\nu\zeta\log\frac{|\mathcal{H}|(\log n)^2}{\delta}}).\\\end{align*}
\end{proof}

\section{Experiment Details}
\subsection{Implementation}\label{subsec:appendix-implement}
All algorithms considered require empirical risk minimization. Instead of optimizing 0-1 loss which is known to be computationally hard, we approximate it by optimizing a squared loss. We use the online gradient descent method in \cite{KL11} for optimizing importance weighted loss functions.

For \idbal, recall that in Algorithm~\ref{alg:main}, we need to find the empirical risk minimizer $\hat{h}_k \gets \arg\min_{h\in V_k} l(h, \tilde{S}_k)$, update the candidate set $V_{k+1} \gets \{ h\in V_k \mid l(h,\tilde{S}_k) \leq l(\hat{h}_k,\tilde{S}_k)+\Delta_k(h,\hat{h}_k)\}$, and check whether $x\in \text{DIS}(V_{k+1})$.

In our experiment, we approximately implement this following Vowpal Wabbit \cite{vw}. More specifically,
\begin{enumerate}
\item Instead of optimizing 0-1 loss which is known to be computationally hard, we use a surrogate loss $l(y,y')=(y-y')^2$.
\item We do not explicitly maintain the candidate set $V_{k+1}$.

\item \label{item:eta} To solve the optimization problem $\min_{h\in V_k} l(h, \tilde{S}_k)=\sum_{(X,\tilde{Y},Z)\in\tilde{S}_k}\frac{\One\{h(X)\neq \tilde{Y}\}Z}{m_k Q_0(X)+n_k Q_k(X)}$, we ignore the constraint $h\in V_k$, and use online gradient descent with stepsize $\sqrt{\frac{\eta}{t+\eta}}$ where $\eta$ is a parameter. The start point for gradient descent is set as $\hat{h}_{k-1}$ the ERM in the last iteration, and the step index $t$ is shared across all iterations (i.e. we do not reset $t$ to 1 in each iteration).

\item \label{item:C} To approximately check whether $x\in \text{DIS}(V_{k+1})$, when the hypothesis space $\mathcal{H}$ is linear classifiers, let $w_k$ be the normal vector for current ERM $\hat{h}_k$, and $a$ be current stepsize. We claim $x\in \text{DIS}(V_{k+1})$ if $\frac{|2w_k^\top x|}{a x^\top x} \leq \sqrt{\frac{C\cdot l(\hat{h}_k,\tilde{S}_k)}{m_k\xi_k+n_k}} + \frac{C\log(m_k+n_k)}{m_k\xi_k+n_k}$ (recall $|\tilde{S}_k|=m_k+n_k$ and $\xi_k = \inf_{x\in\text{DIS}(V_k)}Q_0(x)$) where $C$ is a parameter that captures the model capacity. See \citep{KL11} for the rationale of this approximate disagreement test.

\item $\xi_k = \inf_{x\in\text{DIS}(V_k)}Q_0(x)$ can be approximately estimated with a set of unlabeled samples. This estimate is always an upper bound of the true value of $\xi_k$.
\end{enumerate} 

\dbalw and \dbalwm can be implemented similarly.

\section{Additional Experiment Results}
In this section, we present a table of dataset information and plots of test error curves for each algorithm under each policy and dataset.

We remark that the high error bars in test error curves are largely due to the inherent randomness of training sets since in practice active learning is sensitive to the order of training examples. Similar phenomenon can be observed in previous work \cite{HAHLS15}.

\begin{table}
\centering
\caption{Dataset information.}\label{tab:dataset-info}
\begin{tabular}{lll}
\toprule
Dataset & \# of examples & \# of features \\
\midrule
synthetic & 6000 & 30 \\
letter (U vs P) & 1616 & 16 \\
skin & 245057 & 3 \\
magic & 19020 & 10 \\
covtype & 581012 & 54 \\
mushrooms & 8124 & 112 \\
phishing & 11055 & 68 \\
splice & 3175 & 60 \\
svmguide1 & 4000 & 4 \\
a5a & 6414 & 123 \\
cod-rna & 59535  & 8 \\
german & 1000 & 24 \\
\bottomrule
\end{tabular}
\end{table}

\begin{figure*}
\centering
\begin{subfigure}[b]{0.22\textwidth}
\includegraphics[width=\textwidth]{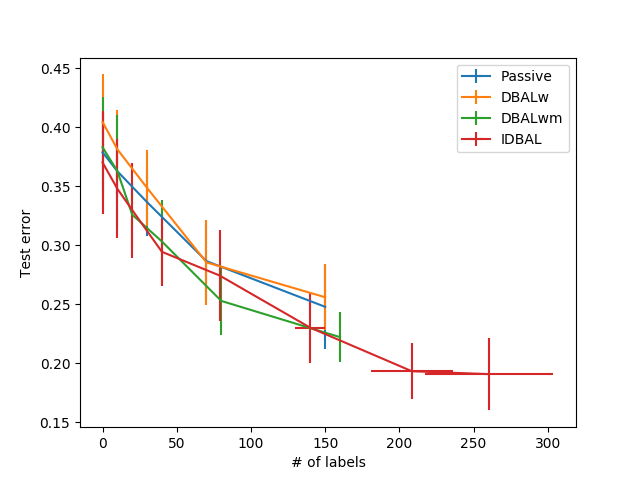}
\caption*{synthetic}
\end{subfigure}
\begin{subfigure}[b]{0.22\textwidth}
\includegraphics[width=\textwidth]{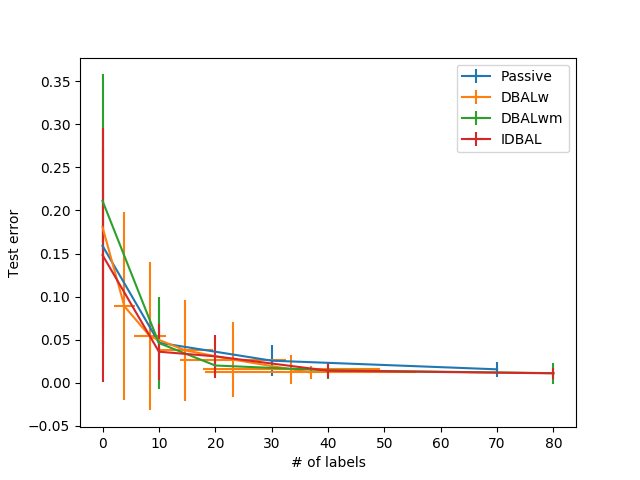}
\caption*{letter}
\end{subfigure}
\begin{subfigure}[b]{0.22\textwidth}
\includegraphics[width=\textwidth]{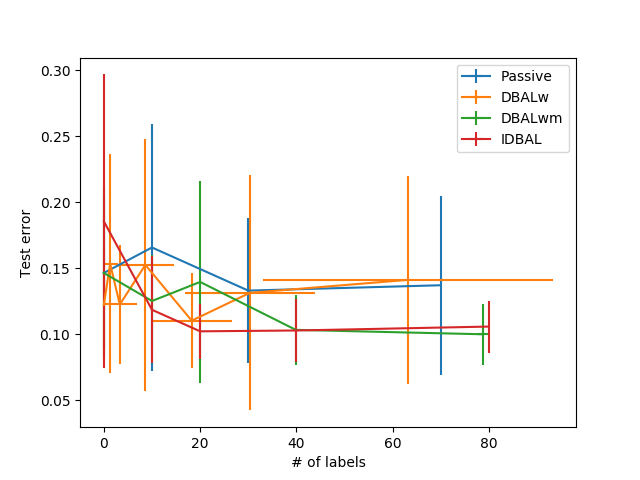}
\caption*{skin}
\end{subfigure}
\begin{subfigure}[b]{0.22\textwidth}
\includegraphics[width=\textwidth]{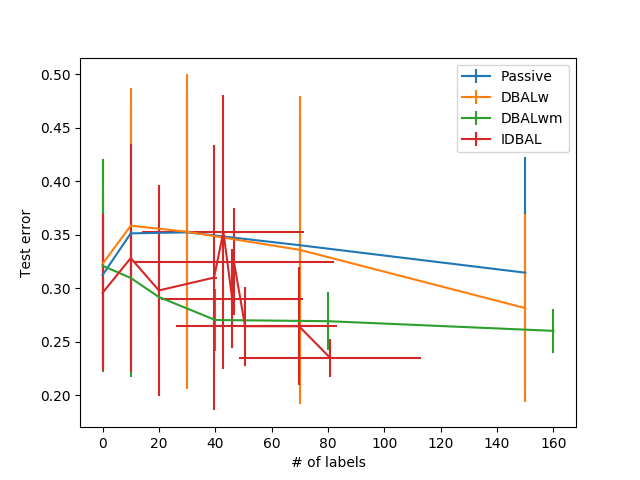}
\caption*{magic}
\end{subfigure}
\begin{subfigure}[b]{0.22\textwidth}
\includegraphics[width=\textwidth]{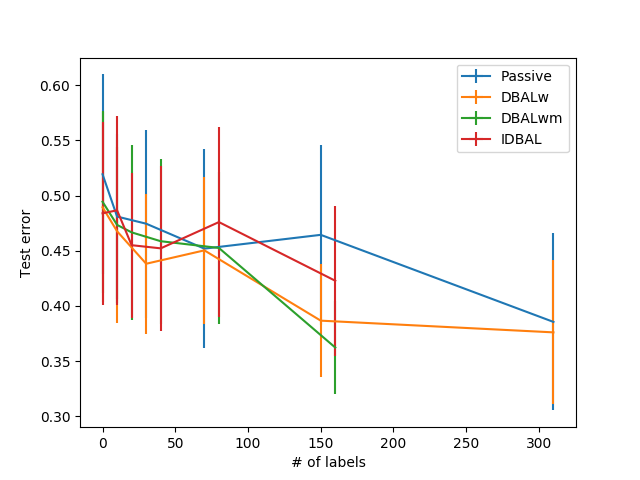}
\caption*{covtype}
\end{subfigure}
\begin{subfigure}[b]{0.22\textwidth}
\includegraphics[width=\textwidth]{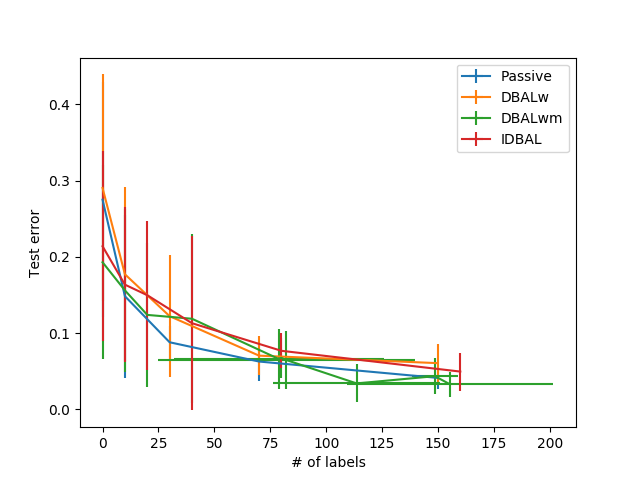}
\caption*{mushrooms}
\end{subfigure}
\begin{subfigure}[b]{0.22\textwidth}
\includegraphics[width=\textwidth]{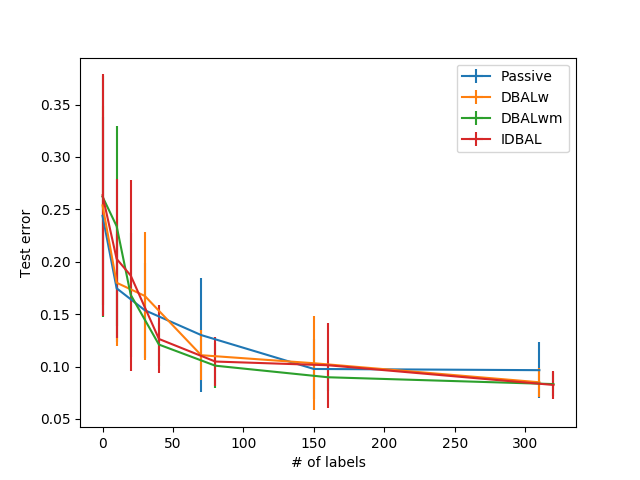}
\caption*{phishing}
\end{subfigure}
\begin{subfigure}[b]{0.22\textwidth}
\includegraphics[width=\textwidth]{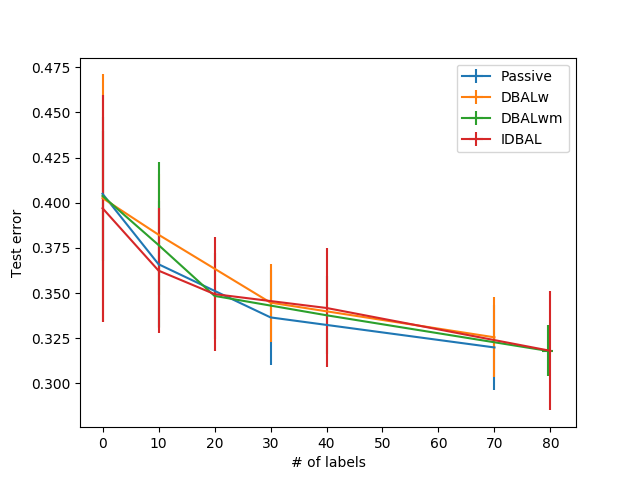}
\caption*{splice}
\end{subfigure}
\begin{subfigure}[b]{0.22\textwidth}
\includegraphics[width=\textwidth]{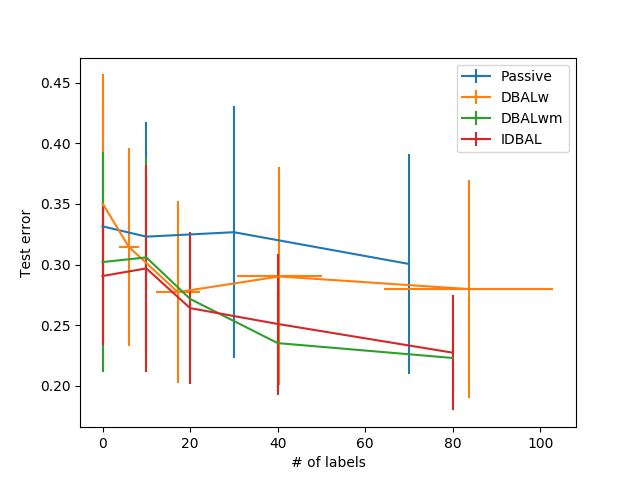}
\caption*{svmguide1}
\end{subfigure}
\begin{subfigure}[b]{0.22\textwidth}
\includegraphics[width=\textwidth]{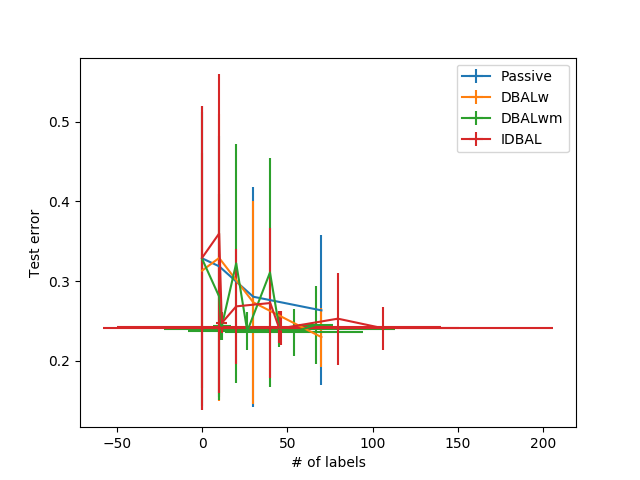}
\caption*{a5a}
\end{subfigure}
\begin{subfigure}[b]{0.22\textwidth}
\includegraphics[width=\textwidth]{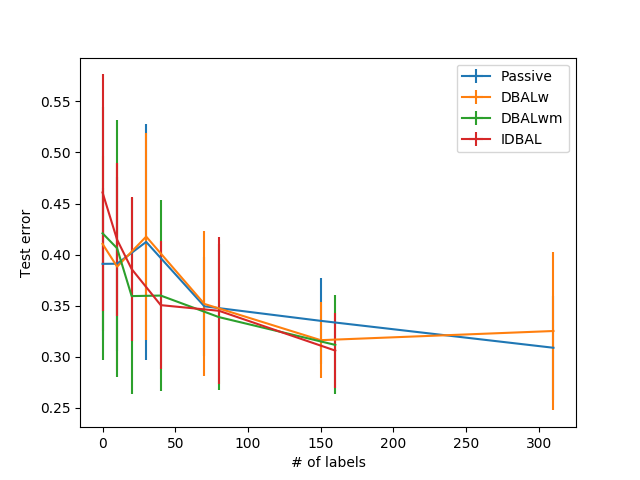}
\caption*{cod-rna}
\end{subfigure}
\begin{subfigure}[b]{0.22\textwidth}
\includegraphics[width=\textwidth]{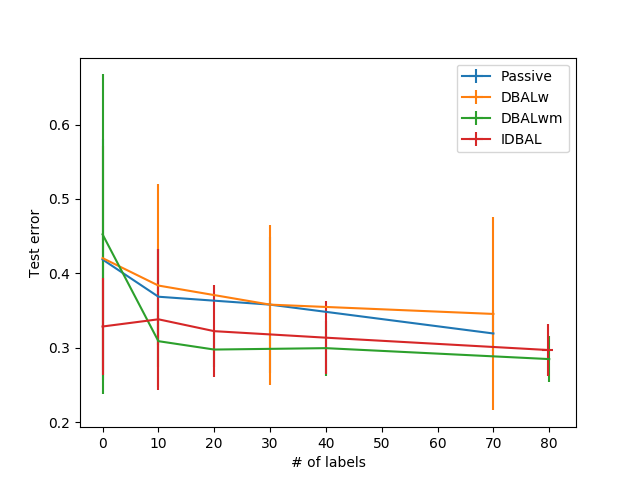}
\caption*{german}
\end{subfigure}
\caption{Test error vs. number of labels under the Identical policy}
\end{figure*}
\begin{figure*}
\centering
\begin{subfigure}[b]{0.22\textwidth}
\includegraphics[width=\textwidth]{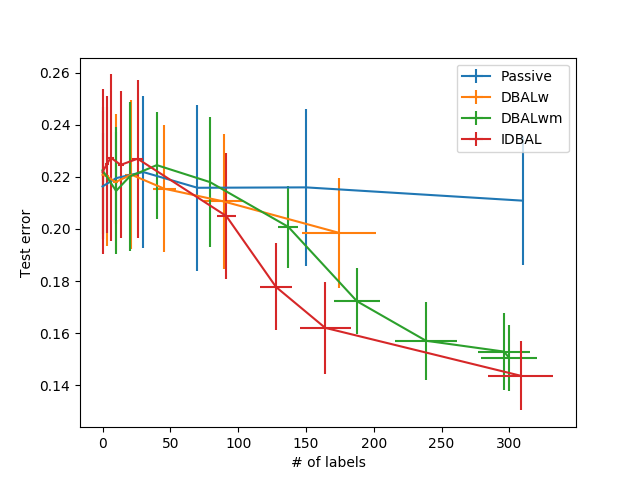}
\caption*{synthetic}
\end{subfigure}
\begin{subfigure}[b]{0.22\textwidth}
\includegraphics[width=\textwidth]{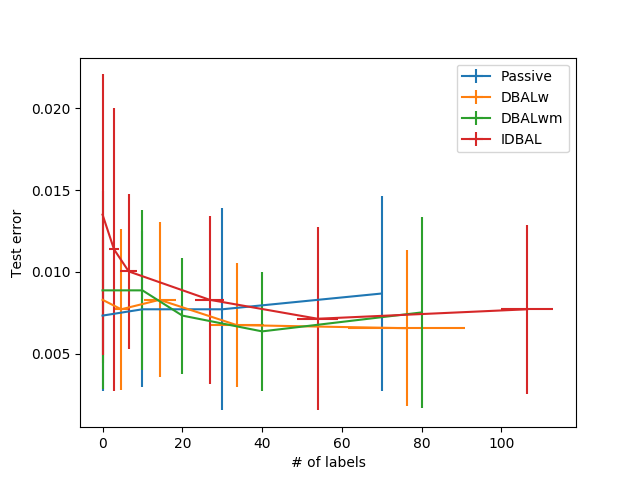}
\caption*{letter}
\end{subfigure}
\begin{subfigure}[b]{0.22\textwidth}
\includegraphics[width=\textwidth]{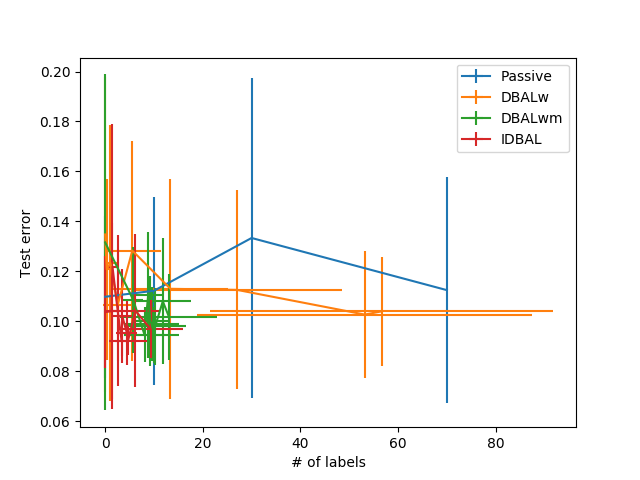}
\caption*{skin}
\end{subfigure}
\begin{subfigure}[b]{0.22\textwidth}
\includegraphics[width=\textwidth]{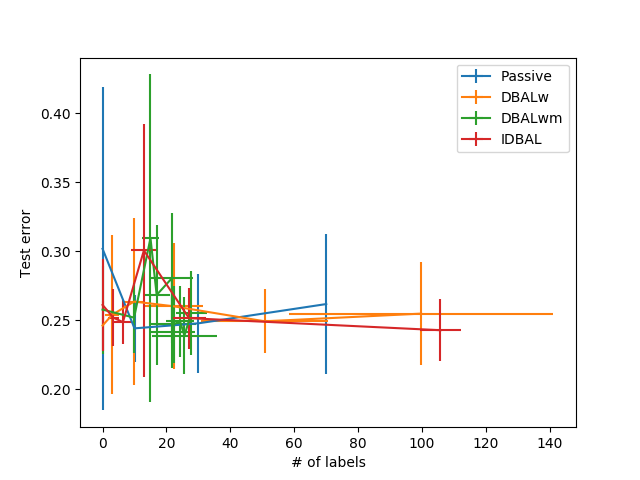}
\caption*{magic}
\end{subfigure}
\begin{subfigure}[b]{0.22\textwidth}
\includegraphics[width=\textwidth]{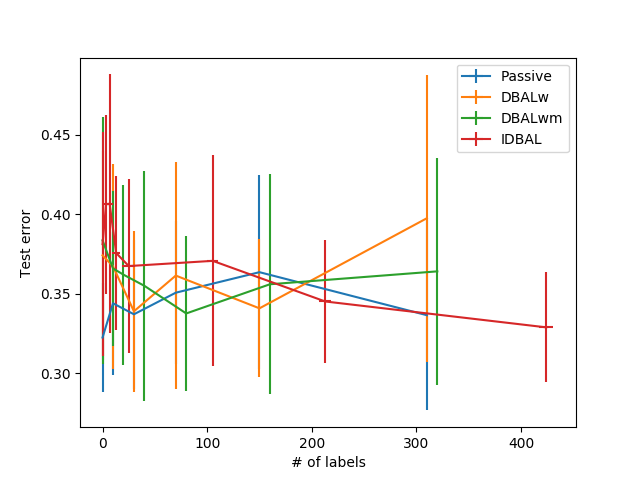}
\caption*{covtype}
\end{subfigure}
\begin{subfigure}[b]{0.22\textwidth}
\includegraphics[width=\textwidth]{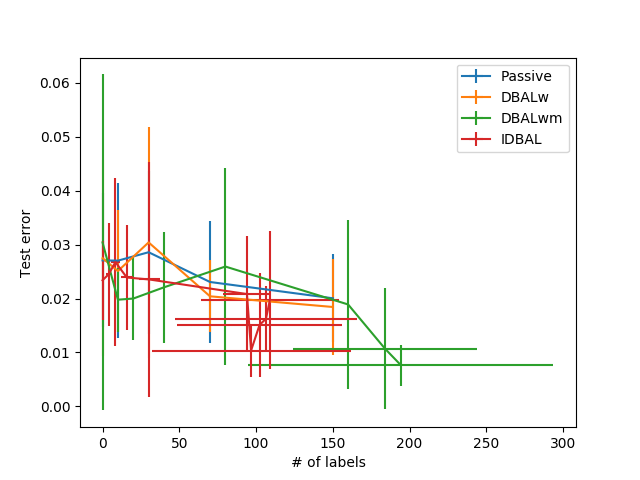}
\caption*{mushrooms}
\end{subfigure}
\begin{subfigure}[b]{0.22\textwidth}
\includegraphics[width=\textwidth]{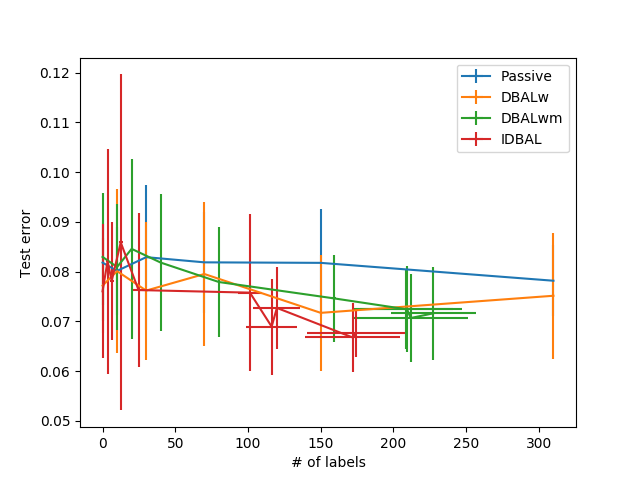}
\caption*{phishing}
\end{subfigure}
\begin{subfigure}[b]{0.22\textwidth}
\includegraphics[width=\textwidth]{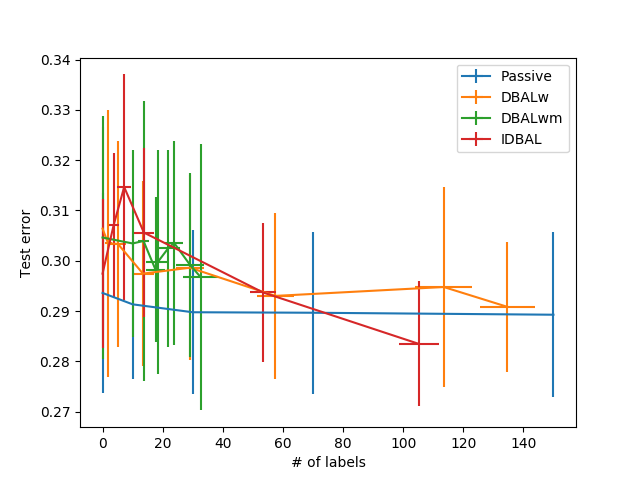}
\caption*{splice}
\end{subfigure}
\begin{subfigure}[b]{0.22\textwidth}
\includegraphics[width=\textwidth]{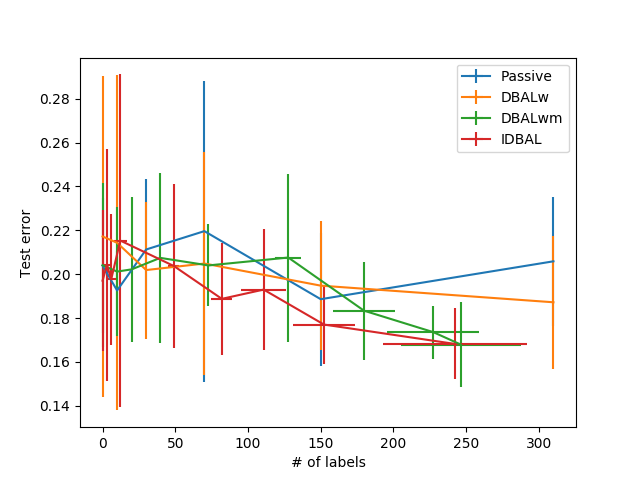}
\caption*{svmguide1}
\end{subfigure}
\begin{subfigure}[b]{0.22\textwidth}
\includegraphics[width=\textwidth]{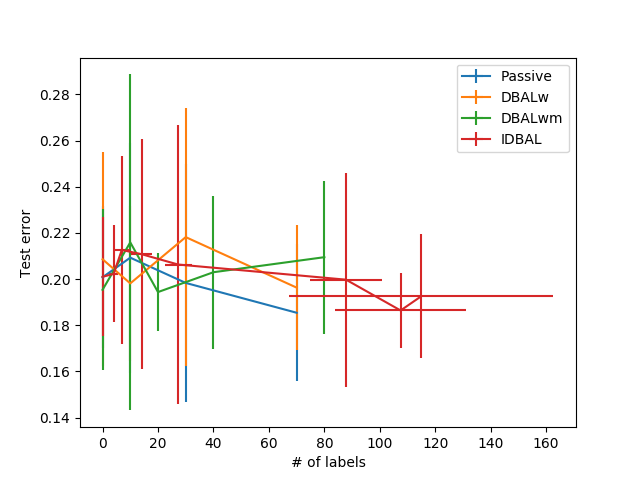}
\caption*{a5a}
\end{subfigure}
\begin{subfigure}[b]{0.22\textwidth}
\includegraphics[width=\textwidth]{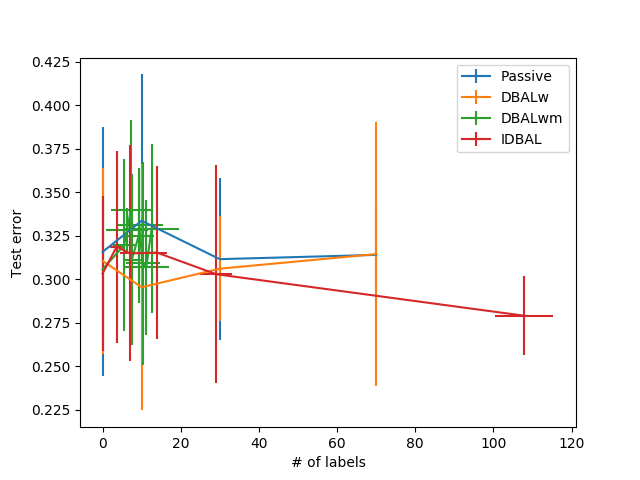}
\caption*{cod-rna}
\end{subfigure}
\begin{subfigure}[b]{0.22\textwidth}
\includegraphics[width=\textwidth]{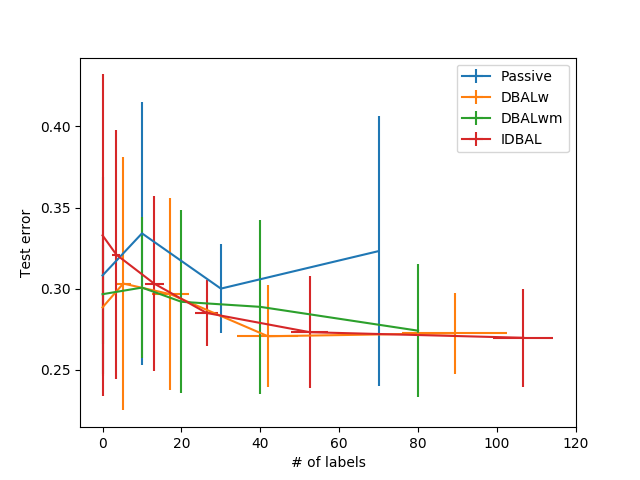}
\caption*{german}
\end{subfigure}
\caption{Test error vs. number of labels under the Uniform policy}
\end{figure*}
\begin{figure*}
\centering
\begin{subfigure}[b]{0.22\textwidth}
\includegraphics[width=\textwidth]{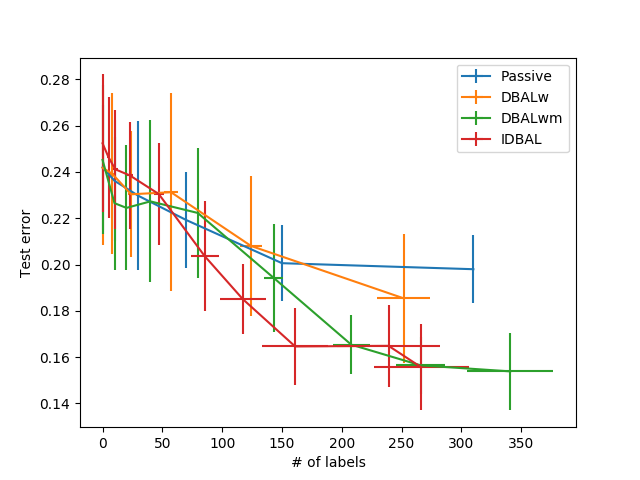}
\caption*{synthetic}
\end{subfigure}
\begin{subfigure}[b]{0.22\textwidth}
\includegraphics[width=\textwidth]{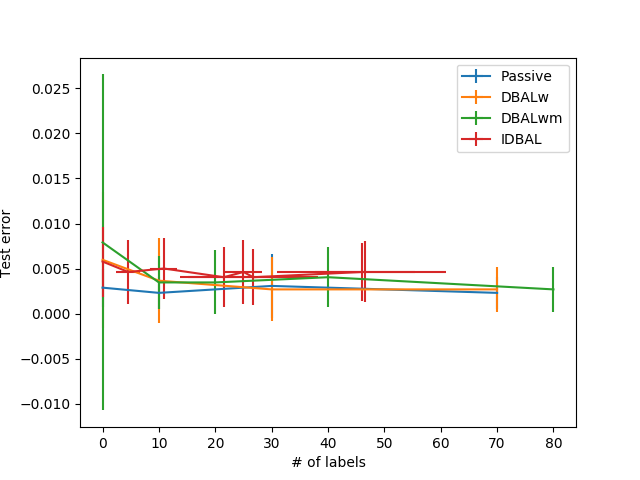}
\caption*{letter}
\end{subfigure}
\begin{subfigure}[b]{0.22\textwidth}
\includegraphics[width=\textwidth]{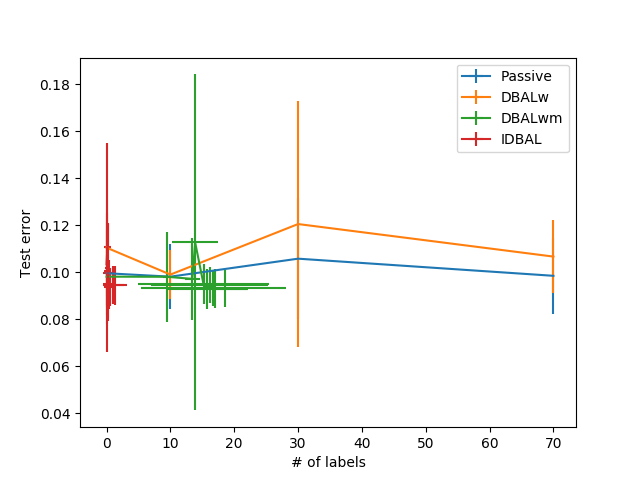}
\caption*{skin}
\end{subfigure}
\begin{subfigure}[b]{0.22\textwidth}
\includegraphics[width=\textwidth]{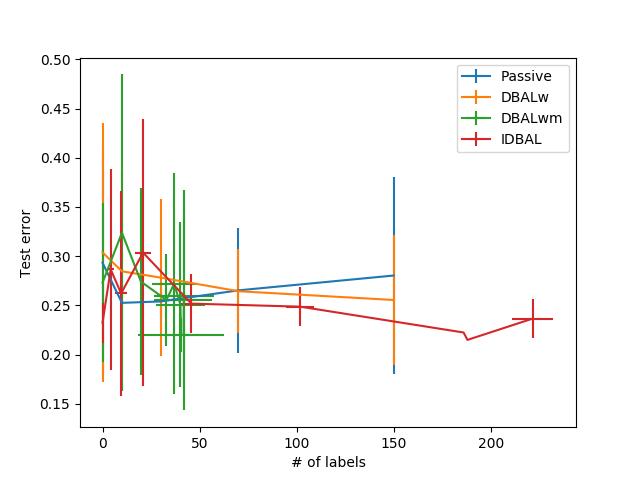}
\caption*{magic}
\end{subfigure}
\begin{subfigure}[b]{0.22\textwidth}
\includegraphics[width=\textwidth]{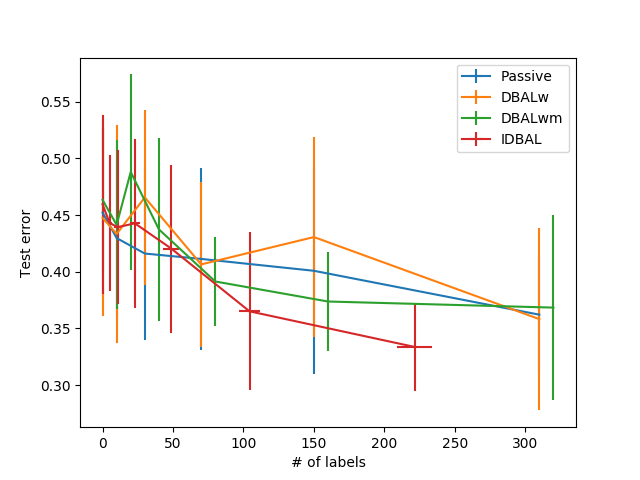}
\caption*{covtype}
\end{subfigure}
\begin{subfigure}[b]{0.22\textwidth}
\includegraphics[width=\textwidth]{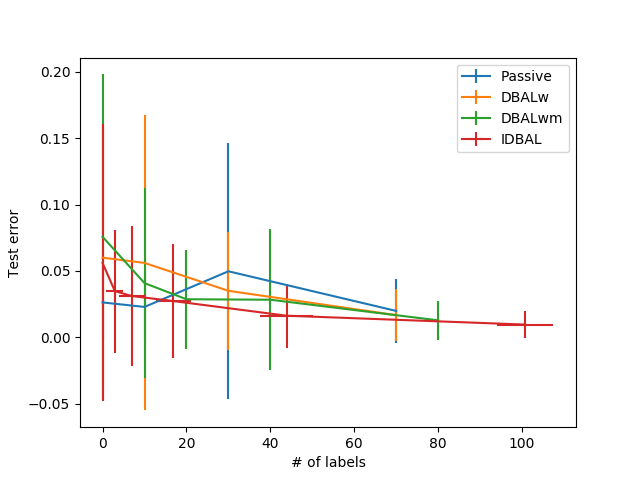}
\caption*{mushrooms}
\end{subfigure}
\begin{subfigure}[b]{0.22\textwidth}
\includegraphics[width=\textwidth]{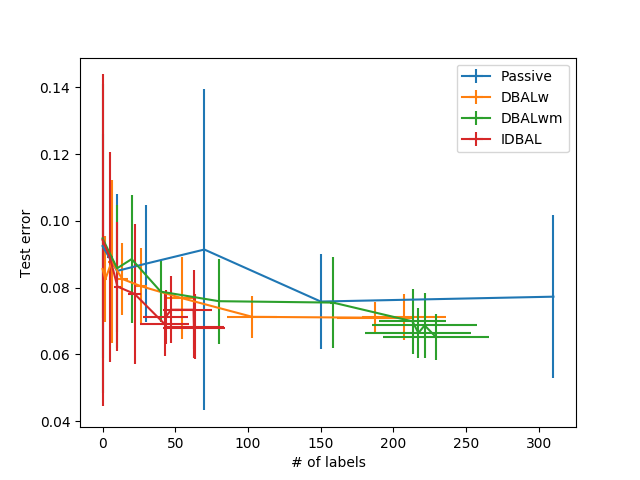}
\caption*{phishing}
\end{subfigure}
\begin{subfigure}[b]{0.22\textwidth}
\includegraphics[width=\textwidth]{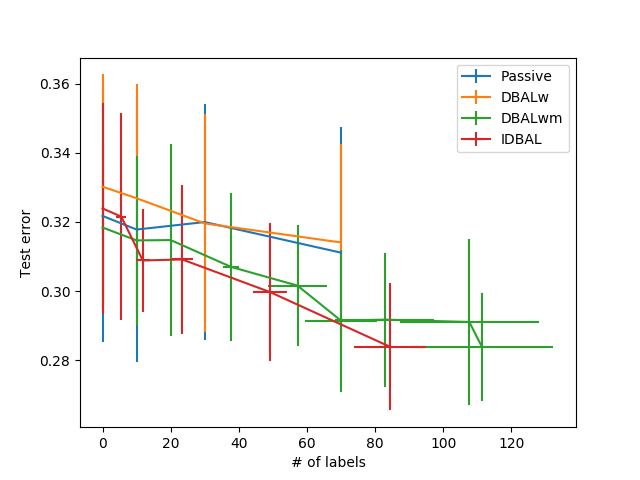}
\caption*{splice}
\end{subfigure}
\begin{subfigure}[b]{0.22\textwidth}
\includegraphics[width=\textwidth]{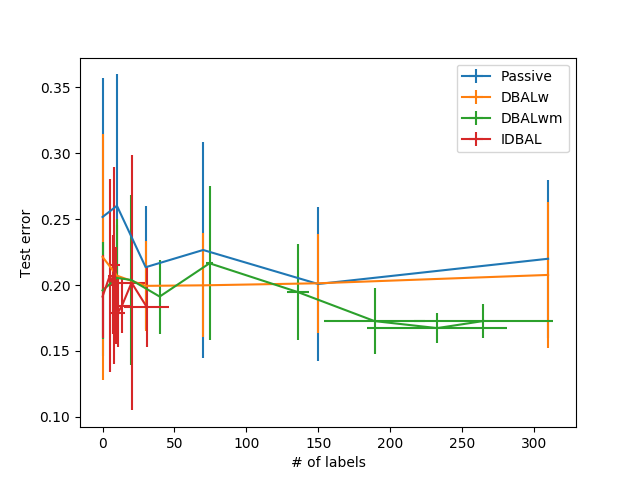}
\caption*{svmguide1}
\end{subfigure}
\begin{subfigure}[b]{0.22\textwidth}
\includegraphics[width=\textwidth]{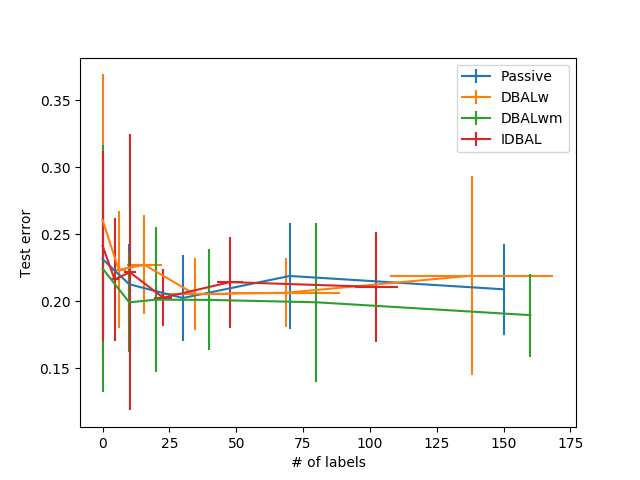}
\caption*{a5a}
\end{subfigure}
\begin{subfigure}[b]{0.22\textwidth}
\includegraphics[width=\textwidth]{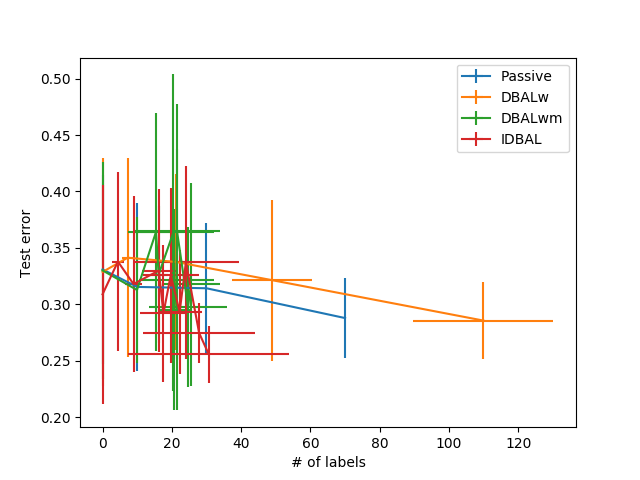}
\caption*{cod-rna}
\end{subfigure}
\begin{subfigure}[b]{0.22\textwidth}
\includegraphics[width=\textwidth]{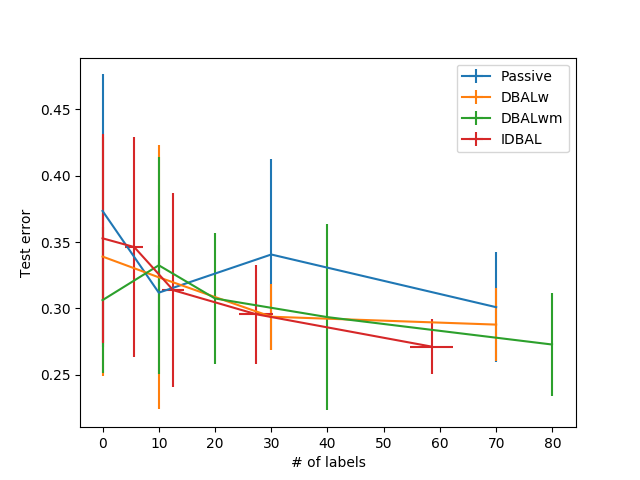}
\caption*{german}
\end{subfigure}
\caption{Test error vs. number of labels under the Uncertainty policy}
\end{figure*}
\begin{figure*}
\centering
\begin{subfigure}[b]{0.22\textwidth}
\includegraphics[width=\textwidth]{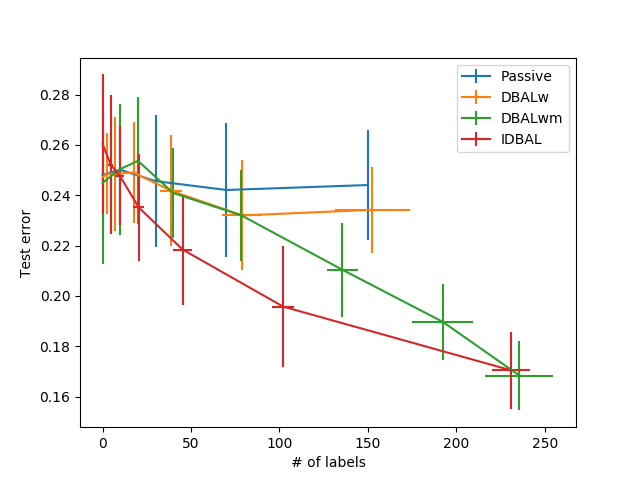}
\caption*{synthetic}
\end{subfigure}
\begin{subfigure}[b]{0.22\textwidth}
\includegraphics[width=\textwidth]{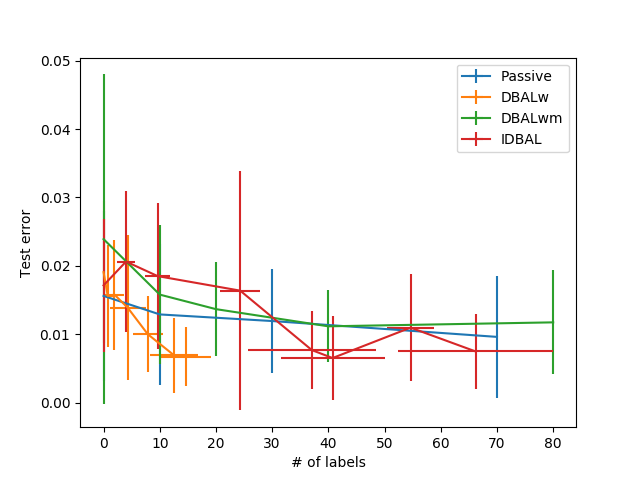}
\caption*{letter}
\end{subfigure}
\begin{subfigure}[b]{0.22\textwidth}
\includegraphics[width=\textwidth]{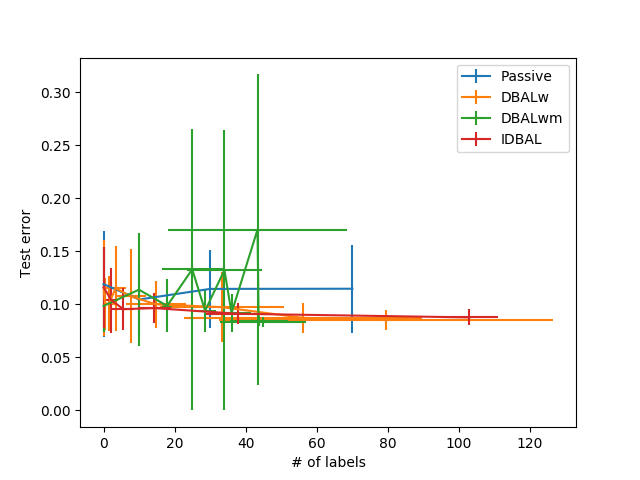}
\caption*{skin}
\end{subfigure}
\begin{subfigure}[b]{0.22\textwidth}
\includegraphics[width=\textwidth]{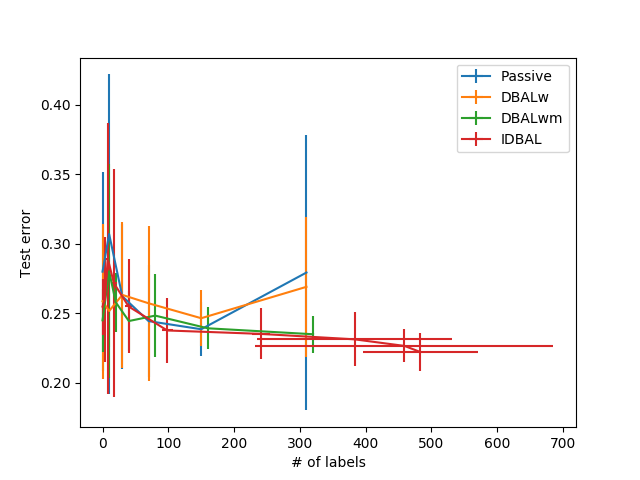}
\caption*{magic}
\end{subfigure}
\begin{subfigure}[b]{0.22\textwidth}
\includegraphics[width=\textwidth]{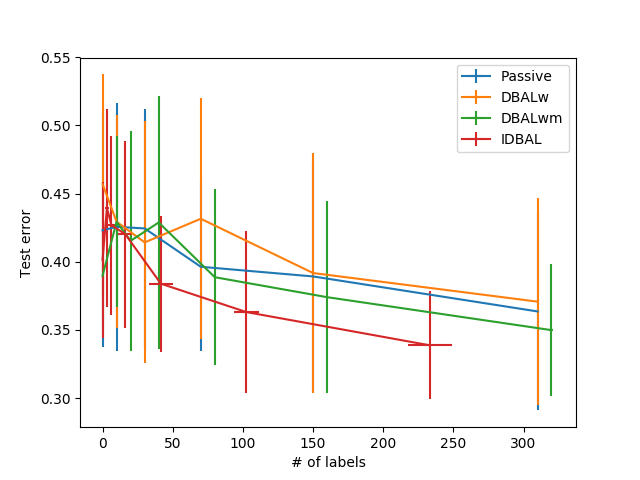}
\caption*{covtype}
\end{subfigure}
\begin{subfigure}[b]{0.22\textwidth}
\includegraphics[width=\textwidth]{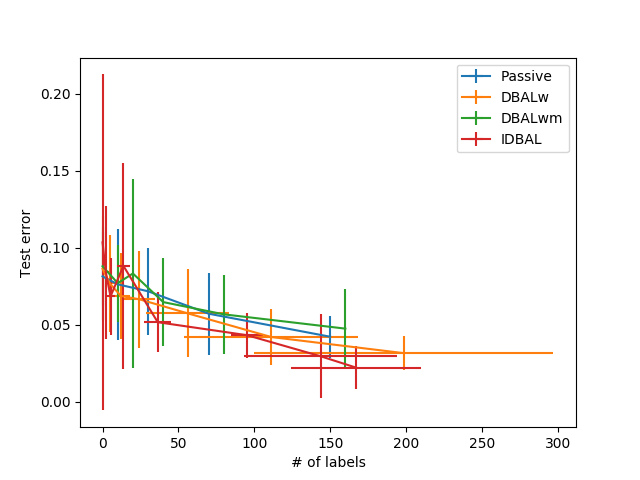}
\caption*{mushrooms}
\end{subfigure}
\begin{subfigure}[b]{0.22\textwidth}
\includegraphics[width=\textwidth]{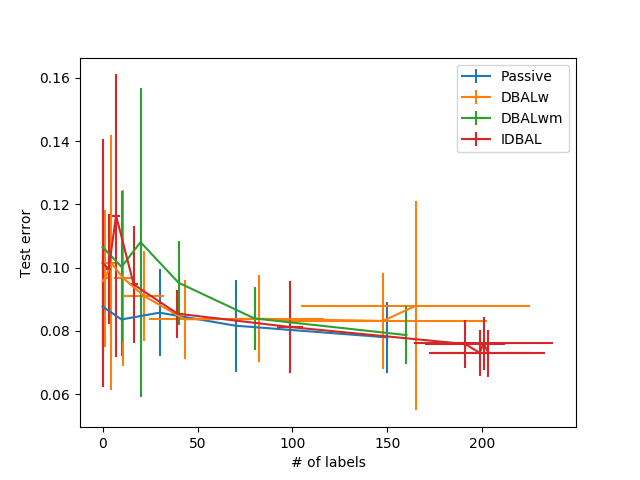}
\caption*{phishing}
\end{subfigure}
\begin{subfigure}[b]{0.22\textwidth}
\includegraphics[width=\textwidth]{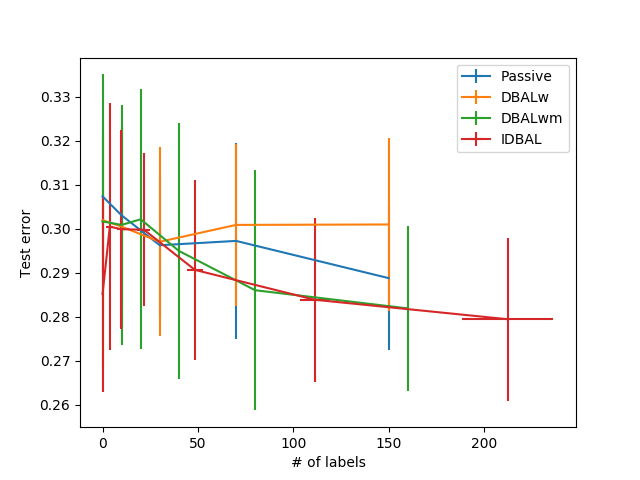}
\caption*{splice}
\end{subfigure}
\begin{subfigure}[b]{0.22\textwidth}
\includegraphics[width=\textwidth]{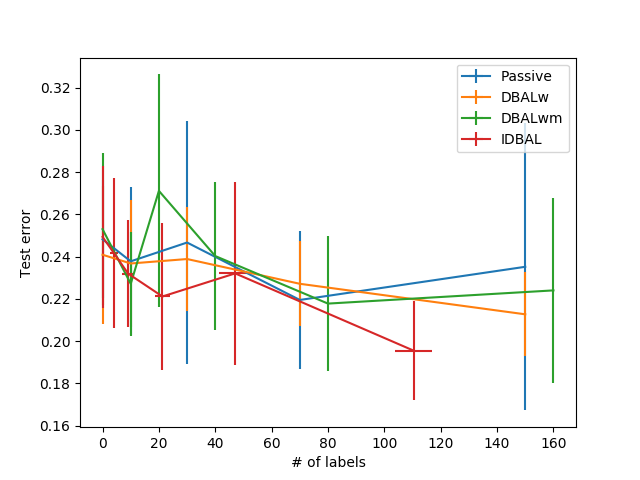}
\caption*{svmguide1}
\end{subfigure}
\begin{subfigure}[b]{0.22\textwidth}
\includegraphics[width=\textwidth]{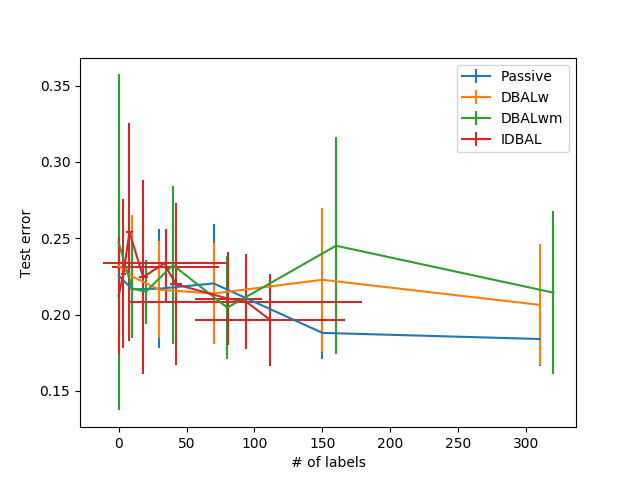}
\caption*{a5a}
\end{subfigure}
\begin{subfigure}[b]{0.22\textwidth}
\includegraphics[width=\textwidth]{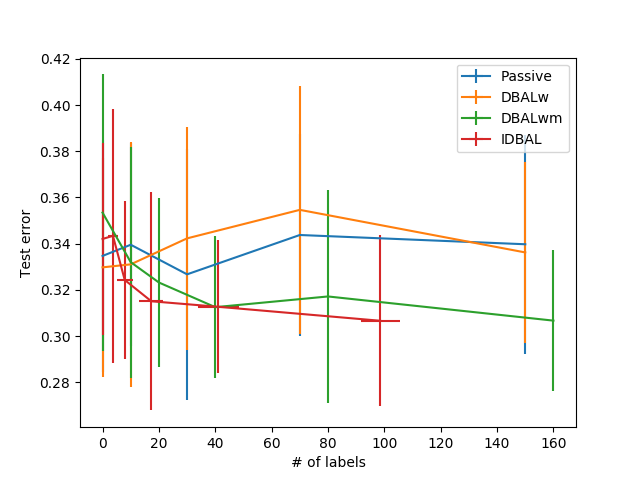}
\caption*{cod-rna}
\end{subfigure}
\begin{subfigure}[b]{0.22\textwidth}
\includegraphics[width=\textwidth]{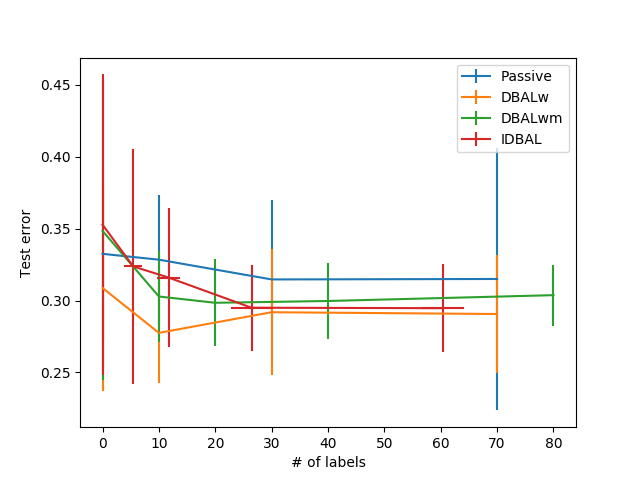}
\caption*{german}
\end{subfigure}
\caption{Test error vs. number of labels under the Certainty policy}
\end{figure*}

\end{document}